\documentclass{article}

\usepackage[preprint,nonatbib]{neurips_2026}  %

\usepackage[utf8]{inputenc} %
\usepackage[T1]{fontenc}    %
\usepackage{hyperref}       %
\usepackage{url}            %
\usepackage{booktabs}       %
\usepackage{amsfonts}       %
\usepackage{nicefrac}       %
\usepackage{microtype}      %
\usepackage{xcolor}         %

\usepackage{graphicx}       %
\usepackage{amsmath}
\usepackage{amssymb}
\usepackage{amsthm}
\usepackage{mathtools}
\usepackage{algorithm}
\usepackage{algorithmic}
\usepackage{multirow}
\usepackage{makecell}
\usepackage{array}
\usepackage{enumitem}
\usepackage{caption}
\usepackage{subcaption}
\usepackage[many]{tcolorbox}
\usepackage{cleveref}
\usepackage{needspace}
\usepackage[numbers,sort,compress,sectionbib]{natbib}

\theoremstyle{plain}
\newtheorem{theorem}{Theorem}
\newtheorem{proposition}[theorem]{Proposition}

\theoremstyle{definition}

\theoremstyle{remark}

\title{Correlation-Aware Exploration for Long-Trace Reasoning}

\workshoptitle{Long-Context Foundation Models}

\newcommand{\blfootnote}[1]{%
  \begingroup
  \renewcommand{\thefootnote}{}%
  \footnotetext{#1}%
  \endgroup
}

\author{%
  \bfseries Junxing Hu\textsuperscript{1} \quad
  Zicheng Zhang\textsuperscript{1,*} \quad
  Zhiqian Zhang\textsuperscript{1} \quad
  Ai Han\textsuperscript{1} \\
  \bfseries Yifeng Zhang\textsuperscript{1} \quad
  Pengzhang Liu\textsuperscript{1} \quad
  Ling Li\textsuperscript{2} \\[0.6em]
  \normalfont \textsuperscript{1}JD.com \qquad \textsuperscript{2}Microsoft
}

\begin{document}

\maketitle

\blfootnote{\textsuperscript{*}Corresponding author.}

\begin{abstract}
Reasoning models trained with Group Relative Policy Optimization (GRPO) now produce chains of thought thousands of tokens long, and recent studies reveal that such training is prone to policy collapse. Entropy regularization is the standard remedy, but it works by pushing probability mass away from the few tokens that carry it and toward the vast low-probability tail of the vocabulary, and it applies this force once for every token of the reasoning trace. This puts the bonus in direct conflict with the entropy decay that convergence requires, and makes its coefficient difficult to set in the long-trace regime. To resolve this tension, we propose Info-GRPO, an information-theoretic framework that cultivates correlation between the policy and a latent prior. Rather than raising entropy everywhere, Info-GRPO builds on the mutual information between a latent variable and the trajectory it induces: augmenting prompts with latent seeds during training lets the model explore diverse policies correlated with the prior, while the seed-conditioned entropy is guided toward convergence. Diversity is therefore carried across whole reasoning traces rather than by per-token randomness, and exploration no longer fights entropy reduction. Experiments show that Info-GRPO outperforms vanilla GRPO and entropy-regularized GRPO across diverse reasoning benchmarks. On AIME24, it improves $Avg@8$ over GRPO by 3.75\%, 1.66\%, 4.16\%, and 4.17\% on four backbones, and it reaches the same majority-vote accuracy with $8\times$ fewer samples.
\end{abstract}

\section{Introduction}

Recent reasoning models~\citep{o12024,o32025,claude37sonnet,comanici2025gemini,guo2025deepseek} have shifted the locus of computation from the prompt to the model's own output: a single answer on a competition mathematics problem now routinely spans thousands of tokens of chain of thought~\citep{Kimi2025,yang2025qwen3}. Central to this progress is Reinforcement Learning with Verifiable Rewards (RLVR)~\citep{lambert2024tulu}, which supplies scalable outcome-based supervision. The context such a policy must manage is therefore largely \emph{self-generated}: a reasoning trace that grows as it is produced and that the training objective must shape end to end. Group Relative Policy Optimization (GRPO)~\citep{shao2024deepseekmath} makes this practical by replacing the critic with group-relative advantage estimation, but it inherits a well-known weakness: as an on-policy method it reinforces existing high-reward paths and collapses onto a narrow set of strategies~\citep{he2025skywork,ma2025learning,hong2018diversity}, which is precisely the failure mode that long, exploratory reasoning is supposed to avoid. Figure~\ref{first_pic} previews the behaviour we obtain instead.

\begin{figure}[!t]
\centering
\includegraphics[width=1.0\textwidth]{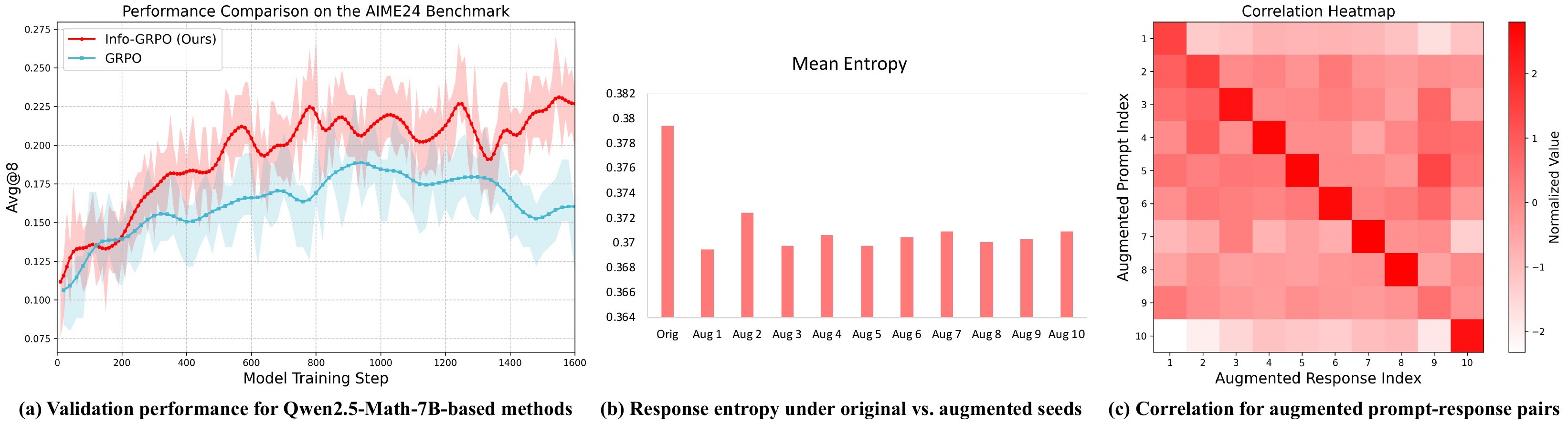}
\caption{
(a) Validation results on Qwen2.5-Math-7B show that Info-GRPO significantly outperforms the GRPO baseline.
(b) Info-GRPO encourages mild high-entropy policies from original prompts and low-entropy policies from seed-augmented prompts. (c) The high correlation between random seeds and their respective outputs shows that generation is reliably steered by the seed it is conditioned on.
}
\label{first_pic}
\vspace{-3mm}
\end{figure}

To counteract this collapse, the conventional remedy is an entropy bonus~\citep{ziebart2008maximum,haarnoja2018soft,yu2025dapo,cui2025entropy,wang2025beyond}. Maximizing entropy means driving the policy toward the uniform distribution, which is to say moving probability mass away from the few tokens the policy currently favours and onto everything else, as Figure~\ref{statistic_pic}(a) illustrates. In a small action space this is a mild and useful perturbation, but a long-trace language policy is the opposite regime, and two factors compound. The entropy of a whole trajectory accumulates over its steps, so a bonus applied to it acts once per token of the trace, and at each of those steps the bonus pushes hardest on the tokens that are least likely, of which a $\sim$152K-token vocabulary~\citep{yang2025qwen3} contains an overwhelming majority. The pressure the bonus places on tokens carrying negligible probability mass therefore accumulates with both the trace length and the vocabulary size, so a growing share of the exploration budget is spent off the reasoning path. Token-level noise is the wrong instrument here for a second, practical reason: in a short generation an anomalous token is cheap, but in an 8K-token derivation every subsequent step is conditioned on it. What long-horizon reasoning needs is the opposite profile: high diversity \emph{between} candidate solution paths, and low uncertainty \emph{within} each one.

In this paper, we propose Info-GRPO, a framework that reframes exploration as a correlation-aware objective rather than a penalty on confidence. Instead of treating exploration and convergence as opposing forces, Info-GRPO conditions the policy on a latent prior and maximizes a generalized form of the mutual information between them. This relocates diversity from the token distribution to the trajectory distribution: different latent priors induce different reasoning paths, while each conditioned policy is free to become confident, so entropy can still fall along the trace as convergence requires (Figure~\ref{first_pic}). The latent prior is used only during training, and inference runs on the original prompt, so deployment is unchanged. In summary, our contributions are as follows.

\begin{itemize}[leftmargin=*]
\item We show that naive entropy regularization in large vocabularies is structurally sensitive to the massive long-tail of tokens, whose contribution accumulates over every step of a long reasoning trace. By framing entropy regularization as a special case of mutual information, we reorient learning toward conditional entropy reduction, which is consistent with convergence.

\item We introduce Info-GRPO, a correlation-aware training paradigm that augments prompts with latent variables and employs a mutual information objective to correlate distinct reasoning strategies with latent priors. This simple, effective design resolves exploration-convergence tension with a coherent trajectory of entropy reduction.

\item We conduct evaluations across diverse benchmarks and models, demonstrating that Info-GRPO consistently outperforms GRPO baselines. Furthermore, latent-conditioned outputs exhibit distinct and interpretable reasoning patterns, providing a novel pathway for steering and analyzing cognitive strategies in LLMs.
\end{itemize}

\section{Related Work}
\label{sec:related}

\textbf{Reinforcement Learning for Long-Trace Reasoning.}
RL significantly improves LLM reasoning~\citep{o12024,o32025,grok2024,qwq2024,guo2025deepseek} and is now the standard route to chains of thought thousands of tokens long. Foundational methods include RLHF~\cite{ouyang2022training} and Verifiable Rewards (RLVR)~\cite{lambert2024tulu,shao2025spurious}, while the value-model-free GRPO~\cite{shao2024deepseekmath} notably enhances stability and efficiency over PPO~\cite{schulman2017proximal}. Recent variants refine advantage estimation (Dr. GRPO~\cite{liu2025understanding}), stabilize MoE training (GSPO~\cite{zheng2025group}), and introduce curricula (Light-R1~\citep{wen2025light}). Others reshape the clipping and sampling pipeline (DAPO~\cite{yu2025dapo}) or reintroduce a value model for advantage estimation (VAPO~\cite{yue2025vapo}, ORZ~\cite{hu2025open}), while a parallel line optimizes policies offline from pre-collected preferences~\citep{rafailov2023direct,meng2024simpo,ethayarajh2024kto}. A growing body of analysis attributes the gains of RLVR to the emergence of reasoning behaviours rather than to correct answers alone~\cite{gandhi2025cognitive}, and further variants act on the length of the chain of thought, as in S-GRPO~\cite{dai2025s}, which rewards conciseness and early termination. However, GRPO-based approaches fundamentally suffer from insufficient exploration and policy collapse~\cite{chen2025dragrpo}. Our work extends GRPO to concurrently improve exploration and exploitation.

\textbf{Entropy Regularization for RL.}
Entropy regularization mitigates premature convergence and insufficient exploration in RL~\citep{hong2018diversity,cui2025entropy}. The idea descends from the maximum-entropy principle~\citep{ziebart2008maximum,toussaint2009robot} and from deep RL methods that add an entropy term to the objective~\citep{mnih2016asynchronous,haarnoja2018soft}. For LLMs, recent works use entropy to drive reasoning~\cite{wang2025beyond}, construct intermediate feedback~\cite{zheng2025first}, augment advantages for diversity~\citep{chen2025dragrpo,cheng2025reasoning}, and control early convergence~\cite{he2025skywork,wang2025reinforcement,cui2025entropy}. A complementary line shapes the entropy \textit{trajectory} along the trace rather than its level: ETR~\citep{etr2026} rewards a monotonically decreasing uncertainty trend to obtain efficient chains of thought. \textit{However, the uncertainty introduced for exploration often weakens policy confidence, and it accumulates over every step of a long trace. We augment entropy-regularized GRPO with a generalized mutual-information regularizer to ease this diversity-confidence contradiction.}

\textbf{Mutual Information for Structured Exploration.}
MI captures variable dependencies and structures representations~\cite{hjelm2018learning,poole2019variational,wen2020mutual,rakelly2021mutual,chen2024learning}, driving representation disentanglement~\cite{chen2016infogan,zhao2017infovae} and contrastive learning~\cite{chen2020simple,zhang2023On,lee2024soft}. In LLMs, MI controls reasoning by downweighting redundant GRPO completions~\cite{chen2025dragrpo}, analyzing trajectory dynamics~\cite{qian2025demystifying}, and optimizing preference alignment~\cite{xiao2025infopo}, enabling structured exploration rather than purely stochastic diversity. Motivated by this, we maximize MI conditioned on a latent variable, so that latent–policy correlation carries diversity across whole reasoning traces rather than per-token randomness. More discussion of related work is provided in Appendix~\ref{Related}.

\section{Background}

In RLVR we model the LLM as a policy $\pi_{\boldsymbol{\theta}}$: from an initial state $s_0$ (the prompt), at step $t$ the state is the action history $s_t=(s_0,a_0,\dots,a_{t-1})$ and the policy $\pi_{\boldsymbol{\theta}}(\cdot\mid s_t)$ emits a token $a_t$. The sequence $\tau=(a_0,a_1,\dots)$ is the trajectory, \emph{i.e.}, the reasoning trace.

\textbf{Group Relative Policy Optimization}~\citep{shao2024deepseekmath} avoids training a critic by estimating the advantage $\hat{A}$ of a trajectory $\tau$ against a sampled group $\mathcal{T}=\{\tau^i\}_{i=1}^{G}$, and plugging it into a PPO-style clipped objective~\citep{schulman2017proximal}:
\begin{equation}
J_{\text{GRPO}}(\boldsymbol{\theta},\mathcal{T}) = \sum_{\tau \in \mathcal{T}} \sum_{t} \min\Big(r_t(\boldsymbol{\theta}) \hat{A}(\tau),\ \text{clip}\big(r_t(\boldsymbol{\theta}), 1 - \epsilon, 1 + \epsilon\big) \hat{A}(\tau) \Big),
    \label{eq:grpo_objective}
\end{equation}
where $r_t(\boldsymbol{\theta}) = \pi_{\boldsymbol{\theta}}(a_t \mid s_t) / \pi_{\boldsymbol{\theta}_{\text{old}}}(a_t \mid s_t)$ is the probability ratio.

\textbf{Policy entropy} measures the uncertainty of the policy. Crucially for what follows, the entropy of a \emph{trajectory} is an accumulation over its steps:
\begin{equation}
    H(\pi_{\boldsymbol{\theta}}(\mathcal{T}|s_0)) := \sum_{\tau\in \mathcal{T}}  \sum_{t} H(\pi_{\boldsymbol{\theta}}(\cdot|s_t)), \  H(\pi_{\boldsymbol{\theta}}(\cdot|s_t)) = - \sum_{a \in V} \pi_{\boldsymbol{\theta}}(a|s_t) \log \pi_{\boldsymbol{\theta}}(a|s_t),
    \label{eq:shannon_entropy}
\end{equation}
where $V$ is the vocabulary. A high entropy value corresponds to a uniform, uncertain policy that encourages exploration, while a low value signifies a deterministic, confident policy geared towards exploitation. Successful RLVR training reduces this quantity as the model commits to high-quality reasoning paths~\citep{cui2025entropy,fu2025deep}. The pathological case, \emph{entropy collapse}, is a sharp early drop that forecloses exploration~\citep{hong2018diversity,he2025skywork,cheng2025reasoning}, countered by a token-level entropy bonus.

\begin{figure}[!t]
\centering
\includegraphics[width=1.0\textwidth]{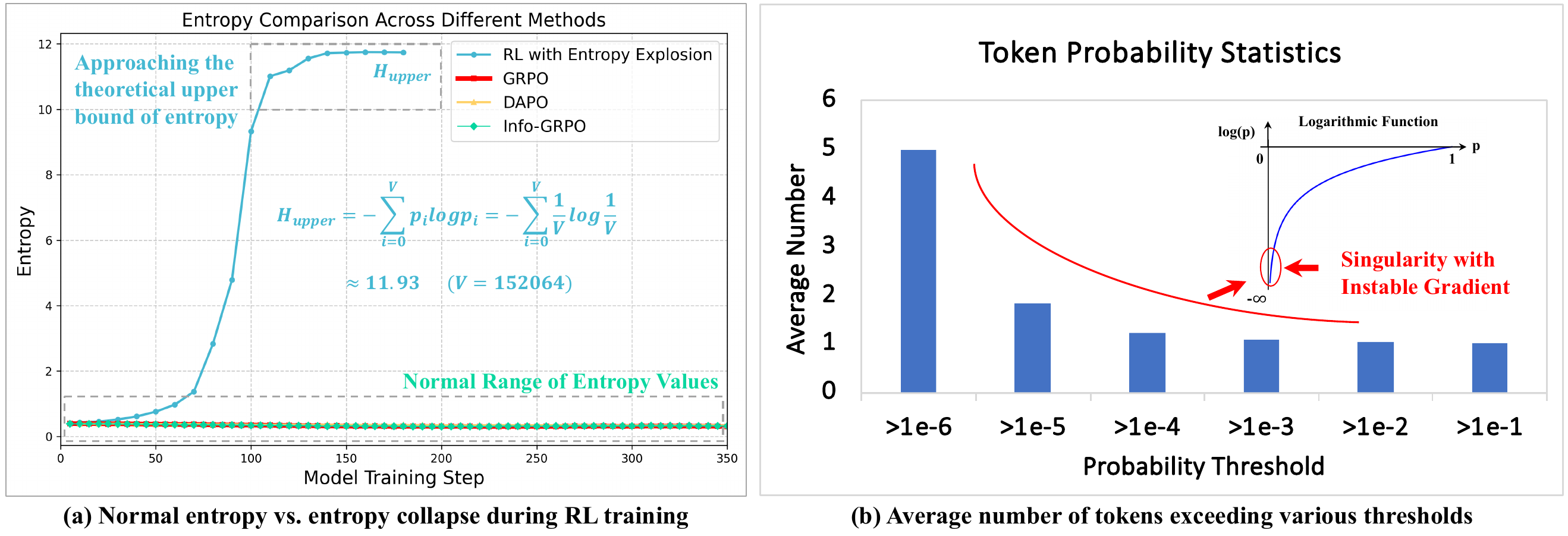}
\caption{
(a) An entropy bonus can degrade the policy into a uniform distribution, reaching its entropy upper bound. (b) This instability is inherent to LLMs, whose vocabularies are dominated by a vast tail of low-probability tokens, on which the entropy objective's logarithmic factor is largest (see Proposition~\ref{thm:main_singularity}). Statistically, in Qwen2.5-7B's 152,064 tokens, fewer than five exceed a $10^{-6}$ probability. This makes direct entropy maximization sensitive to its coefficient at this scale.
}
\label{statistic_pic}
\vspace{-3mm}
\end{figure}

\section{Why Entropy Bonuses Are Fragile at Long Trace Lengths}\label{sec: Singularity Trap}

\textbf{The entropy-regularized objective and its coefficient.}
Exploration is conventionally preserved by adding the bonus of Eq.~(\ref{eq:shannon_entropy}) to the GRPO objective with weight $\beta>0$:
\begin{equation}
    J_{\text{ent}}(\boldsymbol{\theta}) = J_{\text{GRPO}}(\boldsymbol{\theta},\mathcal{T}) + \beta\, H(\pi_{\boldsymbol{\theta}}(\mathcal{T} \mid s_0)).
    \label{eq:entropy_reg}
\end{equation}
Eq.~(\ref{eq:entropy_reg}) localizes the whole exploration--convergence tension in one scalar: the reward term seeks certainty by reducing entropy, the bonus pursues possibility by increasing it, and $\beta$ arbitrates. In practice this knob is hard to set~\citep{cui2025entropy,he2025skywork}: a small $\beta$ barely explores, while a large $\beta$ destabilizes training (Figure~\ref{statistic_pic}(a), at $\beta=0.05$). The sensitivity is structural rather than a matter of tuning.

\begin{proposition}[Tail sensitivity of high-dimensional entropy maximization]
    \label{thm:main_singularity}
    Let $\pi_{\boldsymbol{\theta}}(\cdot\mid s)$ be a policy over a vocabulary $A$ of size $V$, with entropy gradient $\nabla_\theta H(\pi_{\boldsymbol{\theta}}) = - \sum_{a \in A} \nabla_\theta \pi_{\boldsymbol{\theta}}(a\mid s)\cdot(1 + \log \pi_{\boldsymbol{\theta}}(a\mid s))$, in which $(1+\log\pi_{\boldsymbol{\theta}}(a\mid s))$ is the scalar factor contributed by the entropy term alone. Characterizing this factor in probability space, for any threshold $\delta\in(0,1)$ the tail set $A_\delta:=\{a\in A \mid \pi_{\boldsymbol{\theta}}(a\mid s)<\delta\}$ satisfies $|A_\delta| \ge V - 1/\delta$, and the cumulative magnitude of the entropy-specific factors over that tail, the \textbf{Total Tail Instability} (a probability-space diagnostic, see Remark 1), is bounded below by
        \begin{equation}
            \mathrm{TTI} \;=\; \sum_{a \in A_\delta} \big|1 + \log \pi_{\boldsymbol{\theta}}(a\mid s)\big| \;\ge\; \Big(V - \tfrac{1}{\delta}\Big)\,\big|1 + \log \delta\big| ,
            \label{eq:tti}
        \end{equation}
    whenever $V > 1/\delta$, a condition readily met in LLMs.\footnote{Proofs are in Appendix~\ref{Proof} (Theorems~\ref{thm:tail_dominance} and~\ref{thm:gradient_instability}).}
\end{proposition}

\textit{Remark 1.} \textit{TTI measures pressure in probability space}. TTI sums the scalar factor $|1+\log\pi_{\boldsymbol{\theta}}(a\mid s)|$ over the tail. The realized parameter update additionally involves the policy Jacobian and the optimizer. What Eq.~(\ref{eq:tti}) gives is therefore a diagnostic of how the entropy objective \emph{allocates} optimization pressure: the entropy term, alone in an RLVR objective, attaches a factor of magnitude at least $|1+\log\delta|$ to each of at least $V-1/\delta$ tokens that jointly carry vanishing probability mass.

\textit{Remark 2.} \textit{Trace length compounds the effect}. Eq.~(\ref{eq:shannon_entropy}) applies this per-step quantity once for every token of the trace, so for a trajectory of length $T$ the misallocated pressure accumulates as $\mathcal{O}(T\cdot\mathrm{TTI})$, growing with \emph{trace length $\times$ vocabulary size}. Short generations, where entropy bonuses are well behaved, sit at small $T$, while long-form reasoning sits at $T$ in the thousands. A bonus that is benign in the former regime spends a growing share of the exploration budget off the reasoning path in the latter.

\textit{Remark 3.} \textit{The entropy gradient is asymmetrically unstable}. During entropy \emph{maximization} the logarithmic factor amplifies countless near-zero probabilities, pushing the policy toward the uniform distribution (Figure~\ref{statistic_pic}(a)). During entropy \emph{minimization} the same factor is suppressive, driving negligible probabilities toward zero, where the objective is well behaved. Exploration that is obtained by \emph{lowering} a conditional entropy therefore avoids the regime that makes Eq.~(\ref{eq:entropy_reg}) fragile. This asymmetry, rather than the magnitude of any single gradient, is what our method exploits.

\begin{figure}[t]
\centering
\includegraphics[width=0.95\linewidth]{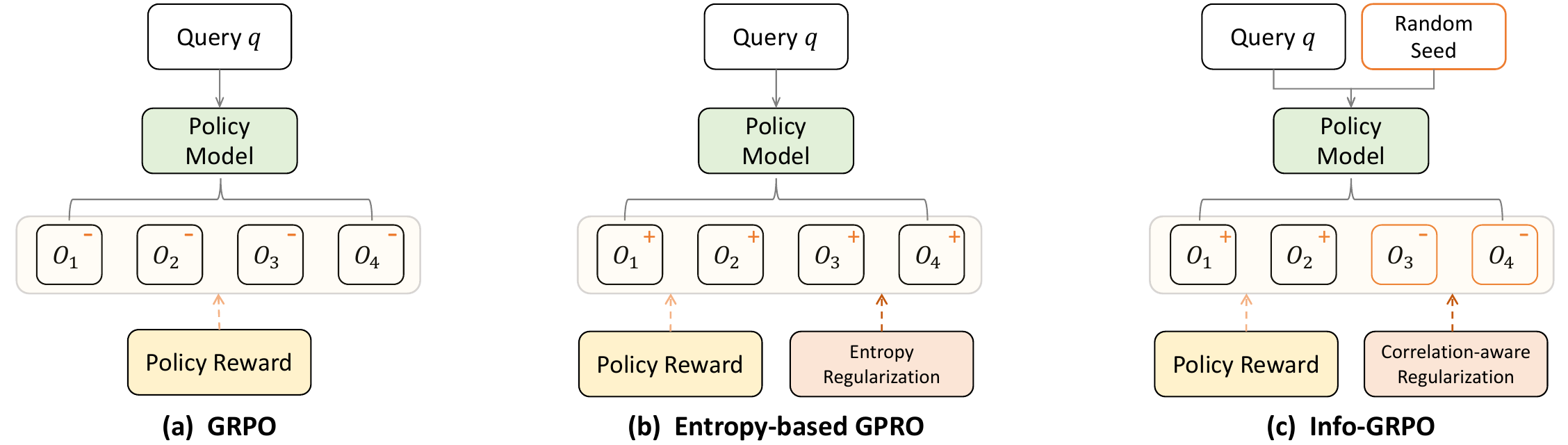}
\caption{An intuitive comparison of exploration methods. Plus (+) and minus (-) signs represent forces that increase or decrease policy entropy. (a) Vanilla GRPO focuses on exploitation (-), often causing policy collapse. (b) Entropy regularization promotes diversity by increasing all entropy (+), but risks collapsing towards a uniform distribution. (c) Info-GRPO balances a mild entropy increase (+) with a strong, seed-conditioned entropy decrease (-). This follows the mutual-information decomposition between seeds and outcomes, using correlation to drive stable exploration.}
\label{fig:algo}
\vspace{-3mm}
\end{figure}

\section{The Info-GRPO Framework}
\label{sec:method_mi}

Motivated by Remark 3, we replace the fragile marginal entropy maximization with its direct information-theoretic extension. This shifts the regularization pressure toward the well-behaved minimization of conditional entropy, mitigating the tail sensitivity identified above while reusing existing computation modules. As shown in Figure~\ref{fig:algo}, conditioning the policy on a latent variable $z \sim p(Z)$ lets the objective balance two competing entropic forces for stable and structured exploration:
\begin{equation}
    \max_{\boldsymbol{\theta}}\ \mathcal{J}_{\text{Info}}(\boldsymbol{\theta}) =  \alpha\, H(\pi_{\boldsymbol{\theta}}(\mathcal{T} \mid s_0)) - H(\pi_{\boldsymbol{\theta}}(\mathcal{T} \mid s_0, z)).
    \label{eq:info_objective}
\end{equation}
The objective is driven by a {consolidation} term, $-H(\pi_{\boldsymbol{\theta}}(\mathcal{T} \mid s_0, z))$, which compels each latent-conditioned policy to converge to a confident strategy. This is counterbalanced by the weighted entropy bonus, in which the \textit{relaxation coefficient} $\alpha \in [0, 1]$ controls how much marginal-entropy pressure is retained: a small positive value anchors exploration in the model's original policy, while values close to $1$ reintroduce the tail sensitivity analyzed in Sec.~\ref{sec: Singularity Trap}.

\textbf{Relation to Mutual Information.}
Our objective in Eq.~(\ref{eq:info_objective}) is derived from the information-theoretic decomposition of the mutual information (MI) between the trajectory $\tau$ and the latent variable $z$,
\begin{equation}
    I(\tau, z \mid s_0) =  \mathbb{E}_{\tau \sim \pi_{\boldsymbol{\theta}}(\cdot|s_0, z),\, z \sim p(Z)} \left[ \log \frac{\pi_{\boldsymbol{\theta}}(\tau, z \mid s_{0})}{p(z)\ \pi_{\boldsymbol{\theta}}(\tau  \mid s_{0}) } \right],
    \label{eq:mi_definition}
\end{equation}
which admits the equivalent form $I(\tau; z \mid s_0) = H(\tau \mid s_0) - H(\tau \mid s_0, z)$. This decomposition, rather than a particular estimator, is the structural principle Info-GRPO retains: the first term maintains diversity at the prompt level, and the second compels each latent-conditioned policy to become confident. Exploration is thereby reframed as maximizing the statistical \textbf{correlation} between latent codes and reasoning paths, cultivating a diverse ensemble of low-entropy policies rather than a single high-entropy one. Diversity now lives between trajectories, not between tokens.

\textbf{Two practical relaxations.}
Eq.~(\ref{eq:info_objective}) coincides with $I(\tau;z\mid s_0)$ at $\alpha=1$ with a seed-averaged marginal, and our implementation departs from that limit in two ways.
\emph{(i) The marginal branch is anchored on unseeded prompts.} Rather than the mixture $p(\tau\mid s_0)=\sum_z p(z)\pi_{\boldsymbol{\theta}}(\tau\mid s_0,z)$, we estimate $H(\tau\mid s_0)$ from trajectories generated \emph{without} a seed. The seed is a training-time device and inference always runs on the original prompt, so the unseeded branch serves as an \emph{inference-aligned anchor}, whereas the mixture would tie it to a distribution never used at test time.
\emph{(ii) The relaxation coefficient satisfies $\alpha<1$.} The marginal term is the tail-sensitive component identified above: $\alpha=1$ restores full pressure to flatten probability mass across the vocabulary and destabilizes training, while an overly small $\alpha$ leaves the anchor too weak. We use $\alpha=0.5$ throughout, and refer to Eq.~(\ref{eq:info_objective}) as a generalized MI regularizer.

\textbf{Implementation.}
Info-GRPO leaves the GRPO loop intact. To match vanilla GRPO's sampling budget ($G=16$), each prompt's generations are split into two equal sub-groups: $\mathcal{T}_{\text{ori}}$ from the base policy $\pi_{\boldsymbol{\theta}}(\cdot\mid s_0)$, and $\mathcal{T}_{\text{aug}}$ from the conditioned policy $\pi_{\boldsymbol{\theta}}(\cdot\mid s_0,z)$, where $z$ is drawn uniformly from a digital set (\emph{e.g.}, $\{1,\dots,10\}$) and appended to $s_0$ by a fixed template (\textit{``Choosing random seed $\{z\}$...''}). Empirically these semantically neutral seeds outperform explicit semantic hints, which impose an \emph{alignment tax} by constraining the reasoning style. The final objective is
\begin{equation}
J_{\text{Info-GRPO}}(\boldsymbol{\theta})  = J_{\text{GRPO}}(\boldsymbol{\theta},\mathcal{T}_{\text{ori}} \cup \mathcal{T}_{\text{aug}}) +\lambda\, \Big(\alpha\, H\big(\pi_{\boldsymbol{\theta}}(\mathcal{T}_{\text{ori}} \mid s_0)\big) - H\big(\pi_{\boldsymbol{\theta}}(\mathcal{T}_{\text{aug}} \mid s_0, z)\big)\Big),
    \label{eq:info-grpo_objective}
\end{equation}
with advantages $\hat{A}(\tau)$ estimated across the merged set $\mathcal{T}_{\text{ori}}\cup\mathcal{T}_{\text{aug}}$, and the two entropy terms estimated by Monte Carlo approximation from $\mathcal{T}_{\text{ori}}$ and $\mathcal{T}_{\text{aug}}$ respectively. Algorithm~\ref{alg_info_grpo} in Appendix~\ref{Details} summarizes the resulting training loop.

\begin{table}[!t]
\centering
\footnotesize
\caption{
Comparisons on multiple backbones and benchmarks ($Avg@8$, \%). `${\dag}$' uses the official model. (\textbf{bold} = best, \underline{underline} = second best). Additional comparisons are provided in Appendix~\ref{Compare}.
}
\label{multi_task_result}
\setlength{\tabcolsep}{3.4pt}
\begin{tabular}{llcccccc}
\toprule
Backbone & Method & AIME24 & AIME25 & AMC & MATH500 & Minerva & \textbf{Average} \\
\midrule
\multirow{7}{*}{\textit{\textbf{Qwen2.5-7B}}} & Base Model & 12.08  & 7.50  & 42.77  & 75.50  & 35.06  & 34.58   \\
&GRPO & 20.42  & 10.83  & 56.48  & \textbf{80.53}  & 37.82  & \underline{41.22}   \\
&DAPO & 18.75  & \underline{11.25}  & 48.95  & 78.13  & 36.53  & 38.72   \\
& {CLIP-Cov} &15.83 & \textbf{12.08} & 51.51 & 79.10 & \underline{38.19} & 39.34  \\
& {KL-Cov} &15.40 & \underline{11.25} & \textbf{59.94} & 74.05 & 35.94 & 39.32  \\
& Skywork-OR1 & \underline{20.83}  & 7.50  & 57.38  & 78.38  & 38.05  & 40.43   \\
& Info-GRPO (Ours) & \textbf{22.08}  & 8.75  & \underline{58.28} & \underline{79.60}  & \textbf{38.37}  & \textbf{41.42}   \\
\midrule
\multirow{8}{*}{\makecell[l]{\textit{\textbf{Qwen2.5-Math-7B}}}} &
Base Model & 11.67  & 8.33  & 48.64  & 82.38  & 36.40  & 37.48   \\
&GRPO & 21.25  & 11.67  & 66.57  & \textbf{86.75}  & \underline{38.97}  & 45.04   \\
&DAPO & 21.25  & 7.92  & \underline{79.52}  & 86.35  & 38.92  & 46.79   \\
&Dr. GRPO ${\dag}$ & \textbf{29.17} & 9.58 & 58.58 & 79.18 & 34.83 & 42.27 \\
&DARS-7B-ET ${\dag}$ & 26.25 & 15.42 & 60.84 & 83.38 & 39.06 & 44.99 \\
&DARS-7B-HW ${\dag}$ & 27.92 & 12.92 & 63.25 & 84.18 & 38.69 & 45.39 \\
&Skywork-OR1 & 22.50  & \underline{13.33}  & \textbf{81.02}  & 84.73  & 38.19 &\underline{47.95} \\
&Info-GRPO (Ours) & \underline{25.00}  & \textbf{15.83}  & 78.46  & \underline{86.63}  & \textbf{39.34}  & \textbf{49.05}   \\
\midrule
\multirow{6}{*}{\makecell[l]{\textit{\textbf{DeepSeek-R1-Distill-Qwen-7B}}}} &
Base Model & 54.58  & 34.17  & 81.63  & 92.03  & 39.89  & 60.46   \\
&GRPO & 57.92  & 34.17  & 82.38  & 93.63  & 44.16  & 62.45   \\
&DAPO & 58.33  & 37.50  & 82.53  & 93.63  & \underline{44.53}  & 63.30   \\
&Light-R1-7B-DS ${\dag}$ & 59.10 & \underline{44.30} & \underline{82.83} & 93.55 & 44.35 & \underline{64.83}    \\
&Skywork-OR1 & \underline{59.58}  & 37.92  & 82.08  & \underline{93.73}  & 44.12  & 63.48   \\
&Info-GRPO (Ours) & \textbf{62.08}  &\textbf{ 45.83}  & \textbf{83.13}  & \textbf{93.98}  & \textbf{44.94}  & \textbf{65.99}  \\
\midrule
\multirow{4}{*}{\makecell[l]{\textit{\textbf{DeepSeek-R1-Distill-Llama-8B}}}} &
Base Model & 50.42 & 32.50 & 79.67 & 85.35 & \underline{37.91} & 57.17\\
&GRPO & \underline{53.75} & 35.42 & \underline{81.02} & 86.23 & 37.22 & 58.73\\
&DAPO & 51.67 & \underline{36.67} & 80.87 & \underline{91.08} & 37.45 & \underline{59.55}\\
&Info-GRPO (Ours) & \textbf{57.92} & \textbf{37.50} & \textbf{83.28} & \textbf{91.45} & \textbf{37.96} & \textbf{61.62}\\
\bottomrule
\end{tabular}
\vspace{-3mm}
\end{table}

\begin{table}[!t]
\centering
\small
\caption{Code benchmark (LiveCodeBench) comparisons using DeepSeek-R1-Distill-Qwen-7B (\%).}
\label{code_result}
\begin{tabular}{lccc}
\toprule
Method & $Pass@1$ & $Pass@8$ & $Avg@8$ \\
\midrule
Base Model & 36.92 & 48.75 & 36.74  \\
GRPO & 37.63 & 50.18 & 38.13 \\
Info-GRPO (Ours) & \textbf{40.86} & \textbf{51.97} & \textbf{40.14}  \\
\bottomrule
\end{tabular}
\vspace{-3mm}
\end{table}

\section{Experiments}\label{EXPERIMENT}

\subsection{Setup and Metrics}
We train four backbones: Qwen2.5-7B~\citep{qwen2.5}, Qwen2.5-Math-7B~\citep{yang2024qwen25math}, DeepSeek-R1-Distill-Qwen-7B and DeepSeek-R1-Distill-Llama-8B~\citep{guo2025deepseek}. For the Qwen2.5 series we use DAPO-Math-17K~\citep{yu2025dapo} (17,917 problems) with a 4,096-token trace budget matching their official configuration. For the DeepSeek-R1-Distill series we use the 48,371-sample set of~\citet{he2025skywork} with an 8,192-token budget, since their reasoning traces are substantially longer. The learning rate is $1e{-}6$ with batch size 128 and $G=16$ rollouts per prompt sampled at temperature 1.0, and we set $\lambda=0.005$ and $\alpha=0.5$ throughout. Every method except those evaluated from officially released checkpoints is retrained by us with the same pipeline and core hyperparameters, on a single node with 8 NVIDIA B200 GPUs. Gradient clipping, loss scaling and the numerical stabilization of log-probabilities are configured identically for every method, including the entropy-regularized baselines, so that the differences we report originate at the objective level rather than from these safeguards. \textbf{Accuracy} is evaluated on AIME24~\citep{li2024numinamath}, AIME25~\citep{balunovic2025matharena}, AMC~\citep{li2024numinamath}, MATH500~\citep{hendrycks2021measuring}, Minerva~\citep{lewkowycz2022solving} and LiveCodeBench~\citep{jain2024livecodebench} (30, 30, 83, 500, 272 and 279 problems), reporting $Pass@1$, $Pass@8$ and $Avg@8$ (the mean of $Pass@1$) over 8 trials at temperature $0.6$~\citep{he2025skywork}. \textbf{Latent--Trajectory Correlation} is the average log-probability of a trace under a given seed, $\text{Corr}(z,\tau\mid s_0)=\tfrac{1}{|\tau|}\sum_t \log \pi(a_t\mid s_0,z,a_{<t})$. \textbf{Semantic Diversity} follows~\citep{chen2025dragrpo}: the average pairwise Euclidean distance between JINA-v2 embeddings~\citep{gunther2023jina} of reasoning paths, aggregated across problems and seeds. Full settings are in Appendix~\ref{Details}.

\subsection{Comparative Experiments}

\textbf{Comparisons on Math Benchmarks across Multiple Backbones.}
In Table~\ref{multi_task_result}, Info-GRPO consistently achieves the best overall performance on both Qwen and Llama backbones, demonstrating robust generalization beyond a single model family. For example, on AIME24, it improves $Avg@8$ over GRPO by 3.75\%, 1.66\%, 4.16\%, and 4.17\% based on four models, respectively. While some methods achieve strong performance on individual benchmarks, \emph{e.g.}, Dr. GRPO on AIME24, they exhibit less consistent behavior across tasks, resulting in weaker overall averages. Info-GRPO and GRPO are nearly tied on Qwen2.5-7B (41.42\% \emph{vs.} 41.22\%), which delimits when our method helps. Correlation-aware exploration converts diversity into accuracy only if the backbone can produce \emph{valid} reasoning paths to diversify over, and on the vanilla Qwen2.5-7B this precondition is weak. The same ceiling appears for the other exploration-driven baseline, Skywork-OR1 (40.43\%). As backbone capacity increases, the overall average over GRPO improves by 4.01\%, 3.54\%, and 2.89\% based on the other three models, respectively. Appendix~\ref{Compare} reports $Pass@1$/$Pass@8$ and a diversity-based analysis of this case.

\textbf{Comparisons on the Code Benchmark.}
Models are trained on a 13.7K-sample dataset~\citep{he2025skywork}. The max token length is also 8K for consistency and efficiency. In Table~\ref{code_result}, Info-GRPO outperforms both the base model and GRPO across all metrics, indicating that the improvements brought by our method extend beyond the mathematics domain and apply to broader reasoning tasks.

\begin{table}[!t]
\centering
\small
\caption{
Comparison with baselines using multiple backbones on AIME24 ($MV@N$, \%).
}
\label{mvn_compare}
\setlength{\tabcolsep}{3.5pt}
\begin{tabular}{llcccccc}
\toprule
Backbone & Samples (N) & 4 & 8 & 16 & 32 & 64 & 128 \\
\midrule
\multirow{2}{*}{\textit{\textbf{Qwen2.5-Math-7B}}}
& GRPO & 23.33 & 16.67 & 16.67 & 23.33 & 23.33 & 26.67 \\
& Info-GRPO (Ours) & 23.33 & 30.00 & 26.67 & 26.67 & 26.67 & 26.67 \\
\midrule
\multirow{2}{*}{\makecell[l]{\textit{\textbf{DeepSeek-R1-Distill-Qwen-7B}}}}
& GRPO & 56.67 & 76.67 & 73.33 & 76.67 & 80.00 & 80.00\\
& Info-GRPO (Ours) & 63.33 & 80.00 & 80.00 & 80.00 & 80.00 & 80.00 \\
\bottomrule
\end{tabular}
\vspace{-3mm}
\end{table}

\textbf{Test-Time Compute under Majority Voting.}
The same structure pays off at inference. Table~\ref{mvn_compare} scales the sampling budget $N$ under majority voting on AIME24. On DeepSeek-R1-Distill-Qwen-7B, Info-GRPO reaches $80.00$ at $N=8$, the accuracy GRPO attains only at $N=64$, an $8\times$ larger budget. On Qwen2.5-Math-7B, GRPO does not reach Info-GRPO's peak of $30.00$ even at $N=128$, although on a 30-problem benchmark that gap is one problem. A diverse but individually confident set of reasoning modes produces candidates that agree sooner, so consensus needs less sampling.

\begin{table}[!t]
\footnotesize
\begin{minipage}[t]{0.31\linewidth}
\centering
\caption{Ablation isolating the vocabulary tail.}
\label{tab_mechanism}
\setlength{\tabcolsep}{3pt}
\begin{tabular}{lc}
\toprule
Entropy scope & AIME24 \\
\midrule
None (GRPO) & 57.92 \\
Full vocab. & 53.23 \\
Head only & 58.13 \\
Seed-aug. MI & \textbf{62.08} \\
\bottomrule
\end{tabular}
\end{minipage}
\hfill
\begin{minipage}[t]{0.32\linewidth}
\centering
\caption{Ablation on the maximum trace budget.}
\label{ablation_length}
\setlength{\tabcolsep}{3pt}
\begin{tabular}{lcc}
\toprule
Budget & AIME24 & AIME25 \\
\midrule
2K & 56.67 & 37.08 \\
3K & 61.25 & 37.92 \\
4K & 60.00 & 40.42 \\
8K & \textbf{62.08} & \textbf{45.83} \\
\bottomrule
\end{tabular}
\end{minipage}
\hfill
\begin{minipage}[t]{0.33\linewidth}
\centering
\caption{Ablation on the regularization coefficient $\lambda$.}
\label{ablation_coef}
\setlength{\tabcolsep}{3pt}
\begin{tabular}{lcc}
\toprule
Coef $\lambda$ & AIME24 & AIME25 \\
\midrule
0.5 & 57.08 & 39.58 \\
0.05 & 58.75 & 40.42 \\
0.005 & \textbf{62.08} & \textbf{45.83} \\
0.002 & \textbf{62.08} & 44.17 \\
\bottomrule
\end{tabular}
\end{minipage}
\vspace{-3mm}
\end{table}

\subsection{Ablation Study}
The ablation tables report $Avg@8$ (\%) on DeepSeek-R1-Distill-Qwen-7B unless stated otherwise.

\textbf{Isolating the Vocabulary Tail.}
We compare three configurations of Eq.~(\ref{eq:entropy_reg}) with only the entropy scope varied: no bonus, the bonus over the full vocabulary, and the bonus restricted to head tokens. Table~\ref{tab_mechanism} shows the predicted pattern: the full-vocabulary bonus degrades performance below the no-entropy baseline ($57.92\rightarrow53.23$), whereas removing the tail from the \emph{same} objective restores it ($58.13$). Since the two variants differ only in whether the tail participates, the tail is the locus of degradation. Info-GRPO reaches $62.08$, well above the tail-masked variant: masking removes an instability, while the MI objective also converts exploration into structured reasoning modes.

\textbf{Ablation on the Trace Budget and the Coefficient $\lambda$.}
Table~\ref{ablation_length} varies the trace budget. Performance improves as the budget grows from 2K to 8K, monotonically on AIME25 ($37.08\rightarrow45.83$) and with the best result at 8K on AIME24: Info-GRPO helps most exactly in the long-trace regime where Remark 2 places the entropy bonus at its most misallocated. Table~\ref{ablation_coef} sweeps $\lambda$: smaller values constrain exploration and improve accuracy, so $\lambda=0.005$ is used for all backbones and benchmarks. Appendix~\ref{Ablation} repeats the grid on disjoint benchmarks and another architecture.

\begin{figure}[t]
\centering
\includegraphics[width=1.0\textwidth]{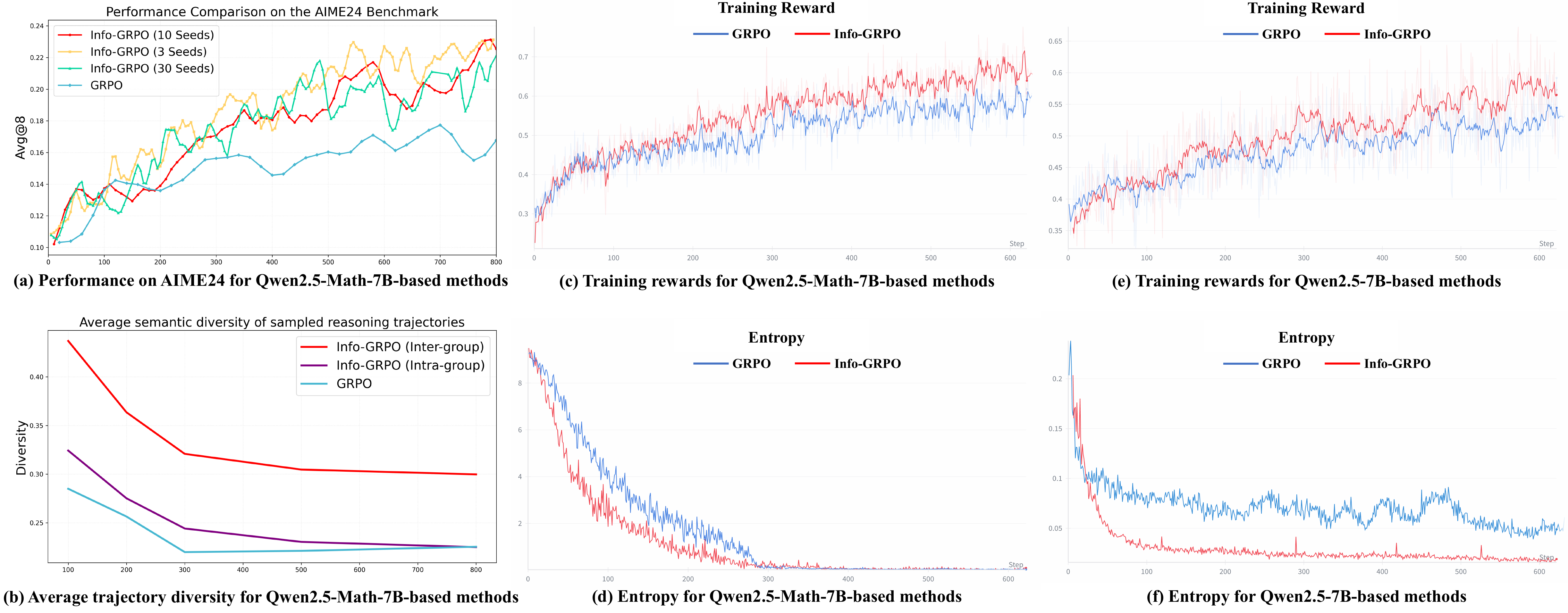}
\caption{
Comprehensive analysis of exploration behavior, policy diversity, and training dynamics.
(a) Validation performance on AIME24 shows that Info-GRPO consistently outperforms GRPO across different numbers of latent seeds, with the performance gain appearing to stabilize around 10 seeds.
(b) Analysis of policy semantic diversity reveals that Info-GRPO (with 10 seeds) maintains high inter-group diversity across conditional seeds and low intra-group diversity within a seed, effectively balancing exploration and exploitation.
(c--f) Info-GRPO exhibits entropy reduction while achieving higher training rewards compared to the GRPO baseline.
}
\label{reward_curve}
\vspace{-3mm}
\end{figure}

\begin{table}[!t]
\footnotesize
\begin{minipage}[t]{0.30\linewidth}
\centering
\caption{Ablation on the relaxation coefficient $\alpha$.}
\label{ablation_alpha}
\setlength{\tabcolsep}{3pt}
\begin{tabular}{lcc}
\toprule
Coef $\alpha$ & AIME24 & AMC \\
\midrule
1.0 (exact MI) & 52.50 & 75.90 \\
0.5 (ours) & \textbf{62.08} & \textbf{83.13} \\
0.05 & 58.33 & 81.93 \\
\bottomrule
\end{tabular}
\end{minipage}
\hfill
\begin{minipage}[t]{0.34\linewidth}
\centering
\caption{Ablation disentangling latent seeds and MI regularization (Qwen2.5-Math-7B).}
\label{ablation_mi}
\setlength{\tabcolsep}{3pt}
\begin{tabular}{lc}
\toprule
Method & AIME24 \\
\midrule
Vanilla GRPO & 45.04 \\
GRPO w/ seeds & 42.23 \\
Info-GRPO & \textbf{49.05} \\
\bottomrule
\end{tabular}
\end{minipage}
\hfill
\begin{minipage}[t]{0.32\linewidth}
\centering
\caption{Ablation on structured latent variables (Qwen2.5-Math-7B).}
\label{ablation_seed_type}
\setlength{\tabcolsep}{3pt}
\begin{tabular}{lcc}
\toprule
Seed type & Diversity & AIME24 \\
\midrule
Semantic & \textbf{0.307} & 23.33 \\
Digital & 0.300 & \textbf{25.00} \\
\bottomrule
\end{tabular}
\end{minipage}
\vspace{-3mm}
\end{table}

\textbf{Ablation on the Relaxation Coefficient $\alpha$.}
Table~\ref{ablation_alpha} shows why the exact-MI setting is not the practical one. At $\alpha=1.0$, where Eq.~(\ref{eq:info_objective}) coincides with mutual information, training is unstable and AIME24 accuracy falls to $52.50$, below vanilla GRPO, exactly as the tail sensitivity of the marginal term predicts. At $\alpha=0.05$ the anchor is too weak. The intermediate $\alpha=0.5$ is best on both and is used throughout.

\textbf{Ablation on Latent Seeds and MI Regularization.}
To disentangle prompt augmentation from the MI objective, we test standard GRPO with seed-augmented prompts (Table~\ref{ablation_mi}). Without MI regularization, appending seeds acts purely as noise, degrading $Avg@8$ from 45.04\% to 42.23\%. The augmented prompt is only a vehicle, and the MI-guided regularizer is what turns the latent variable into a usable control signal. Replacing neutral digital seeds with semantic heuristics (Table~\ref{ablation_seed_type}) marginally increases diversity but degrades efficacy, exposing an \emph{alignment tax}: explicit linguistic constraints push the model off its native reasoning abstractions.

\textbf{Analysis of Latent Seeds and Policy Diversity.}
Varying the seed count (Figure~\ref{reward_curve}(a)), Info-GRPO consistently outperforms GRPO, with gains stabilizing around 10 seeds. Figure~\ref{reward_curve}(b) shows markedly higher inter-group semantic diversity across seeds and lower intra-group diversity within a seed: different latent variables steer the model toward different strategies while it remains confident within each, avoiding the policy collapse of vanilla GRPO. Figure~\ref{first_pic}(c) visualizes the seed--trajectory correlation as a block-diagonal structure, confirming that the policy uses the latent seed. Appendix~\ref{Case} gives qualitative cases in math and code.

\textbf{Analysis of the Training Entropy.}
Figure~\ref{reward_curve}(c--f) compares entropy and training rewards. Info-GRPO shows faster entropy reduction and steeper reward growth, indicating more stable optimization. Suppressing token-level randomness does not suppress exploration: it has been relocated to the trajectory distribution, which is what long-horizon reasoning needs, since per-token noise would let one anomalous token corrupt all that follows.

\Needspace*{6\baselineskip}
\section{Conclusion}

This paper examines the conflict between naive entropy regularization and convergence in long-trace reasoning, and traces it to the tail sensitivity of direct marginal entropy maximization in vast token spaces, a cost paid once per generated token and quantified in probability space by our TTI diagnostic, then verified by a controlled tail-masking ablation. We introduce Info-GRPO, an information-theoretic framework that eases this tension through latent-variable augmentation and a generalized MI-guided regularizer. Experiments show consistent improvements over GRPO baselines across multiple benchmarks and model architectures. Info-GRPO's advantage stems from correlation-aware exploration, as evidenced by seed-conditioned reasoning modes and stable training dynamics.

\section*{Limitations}

\textbf{Trace lengths beyond 8K.} Our longest trace budget is 8,192 tokens, set by the official configurations of the DeepSeek-R1-Distill backbones (4,096 for the Qwen2.5 series). Remark 2 predicts that the pressure an entropy bonus misallocates grows as $\mathcal{O}(T\cdot\mathrm{TTI})$, so the effect motivating Info-GRPO should be at least as pronounced at 32K or 128K trace budgets. We have not verified this: doing so requires retraining at those budgets, which is beyond our compute. The claims in this paper are therefore supported up to 8K.

\textbf{Statistical scope.} We report $Avg@8$, the mean of $Pass@1$ over 8 independent decoding trials, and we do not report a dispersion estimate for it. The informative measurement would repeat the full RL training run with independent random seeds, which is prohibitively expensive at our scale. On 30-problem benchmarks a one-problem difference is worth roughly $3.3$ points, so our conclusions there rest on the consistency of the trend across four backbones, six benchmarks, three metrics, and the larger MATH500 (500), Minerva (272) and LiveCodeBench (279) test sets.

\textbf{What the latent seed contributes.} How much of the gain is attributable to the seed itself remains open. Our ablations separate the MI regularizer from prompt augmentation, but not the effect of varying the seed from that of the conditional-entropy term, which a matched control holding the seed fixed across all augmented prompts would measure. How seeds should be scheduled over training, and whether a seed can be chosen per problem rather than sampled uniformly, are equally unexplored.

\textbf{Latent space complexity.} We rely on a flat, unstructured discrete latent space. As Table~\ref{ablation_seed_type} shows, neutral digital seeds avoid the \emph{alignment tax} incurred by semantic heuristics, but their unstructured nature limits the granularity of control. Hierarchical or continuous latent spaces allowing compositional control over the reasoning process are a substantial standalone problem.

\textbf{Evaluation scope and verifiable rewards.} To isolate the algorithmic effect of the MI objective without confounds from noisy or hackable reward models, evaluation is restricted to domains with strict verifiable rewards, namely mathematics and code. Generalization to multi-step tool use, environments with varying reward sparsity, and fully open-ended reasoning therefore remains untested.

\textbf{Boundary conditions.} As reported in Sec.~\ref{EXPERIMENT} and detailed in Appendix~\ref{Compare}, the method does not help when the backbone cannot produce valid reasoning paths to diversify over: on Qwen2.5-7B, Info-GRPO reaches the highest trajectory diversity of all methods on AIME25 without the highest accuracy. Adaptive modulation of the exploration term for such backbones is left to future work.

{\small
\bibliographystyle{unsrtnat}
\bibliography{info_cite}
}

\newpage
\appendix

\section{Appendix}\label{Supp}

This supplementary material is organized as follows to support the main manuscript:
\begin{itemize}[leftmargin=*]
    \item \textbf{Section~\ref{Proof}:} Additional theoretical proofs and derivations.
    \item \textbf{Section~\ref{Related}:} Extended discussion of related work.
    \item \textbf{Section~\ref{Details}:} More training hyperparameters and implementation details.
    \item \textbf{Section~\ref{Compare}:} Additional comparative results on further metrics, baselines and model series.
    \item \textbf{Section~\ref{Ablation}:} Extended ablation studies complementing those in the main text.
    \item \textbf{Section~\ref{Case}:} Qualitative case analyses demonstrating reasoning diversity.
    \item \textbf{Section~\ref{LLM}:} Disclosure regarding the use of AI assistants in manuscript preparation.
\end{itemize}

\subsection{Proof}\label{Proof}

\setcounter{theorem}{0}
\begin{theorem}[Quantitative Bound of the Tail Set]
\label{thm:tail_dominance}
Let $\pi_\theta(\cdot|s)$ be a policy distribution over a discrete vocabulary $A$ of size $V$. For any probability threshold $\delta \in (0, 1)$, let the tail set $A_\delta$ be defined as the set of tokens whose probability is less than $\delta$:
\begin{equation}
    A_\delta := \{ a \in A \mid \pi_\theta(a|s) < \delta \}.
\end{equation}
The size of this tail set, $|A_\delta|$, is lower-bounded as follows:
\begin{equation}
     |A_\delta| \ge V - \frac{1}{\delta} .
\end{equation}
\end{theorem}

\begin{proof}

\textbf{1. Bounding the Size of the Head.}
First, we partition the vocabulary $A$ into the tail set $A_\delta$ and its complement, the head set $A_\delta^c = A \setminus A_\delta$. By definition, for any token $a \in A_\delta^c$, its probability is bounded below by $\delta$:
\begin{equation}
    \forall a \in A_\delta^c, \quad \pi_\theta(a|s) \ge \delta.
\end{equation}
The total probability mass is constrained by the normalization axiom of probability distributions:
\begin{equation}
    \sum_{a \in A} \pi_\theta(a|s) = 1.
\end{equation}
We can decompose this sum over the head and tail sets:
\begin{equation}
    \sum_{a \in A_\delta^c} \pi_\theta(a|s) + \sum_{a \in A_\delta} \pi_\theta(a|s) = 1.
\end{equation}
Since probabilities are non-negative, $\sum_{a \in A_\delta} \pi_\theta(a|s) \ge 0$. This implies an upper bound on the probability mass contained within the head set:
\begin{equation}
    \sum_{a \in A_\delta^c} \pi_\theta(a|s) \le 1.
\end{equation}
Combining this with the lower bound on individual token probabilities in the head set, we get:
\begin{equation}
    |A_\delta^c| \cdot \delta \le \sum_{a \in A_\delta^c} \pi_\theta(a|s) \le 1.
\end{equation}
From this, we derive a strict upper bound on the size of the head set, $|A_\delta^c|$:
\begin{equation}
    |A_\delta^c| \le \frac{1}{\delta}.
\end{equation}
This result is critical: it shows that the number of {``}high-probability" tokens is independent of the vocabulary size $V$ and is solely limited by the chosen threshold $\delta$.

\textbf{2. Deriving the Lower Bound for the Tail Size.}
The size of the tail set is simply the total vocabulary size minus the size of the head set:
\begin{equation}
    |A_\delta| = V - |A_\delta^c|.
\end{equation}
Substituting our upper bound for $|A_\delta^c|$, we obtain the lower bound for the tail size:
\begin{equation}
    |A_\delta| \ge V - \frac{1}{\delta}.
\end{equation}
This proves the first part of the theorem. As $V \to \infty$, the term $1/\delta$ becomes negligible, and thus $|A_\delta| \approx V$. The vast majority of tokens must lie in the tail.

\end{proof}

\textit{Remark 1.}
    According to our result, for large $V$, the gradient sum is dominated by at least $V - 1/\delta$ such terms. For instance, in an LLM with $V=150,000$ and a small probability threshold of $\delta=10^{-8}$, the head can contain at most $10^8$ tokens (a loose bound), but the tail is guaranteed to contain at least $150,000 - 10^8$, which is nonsensical unless the head is much smaller. A more realistic scenario, if we assume the top 100 tokens hold significant probability mass, we can set $|A_\delta^c| = 100$, implying $\delta \approx 0.01$ at most. Then $|A_\delta| \ge 150,000 - 100 = 149,900$.

\begin{theorem}[Quantitative Bound on the Entropy-Specific Tail Factor]
\label{thm:gradient_instability}
Let $\pi_\theta(\cdot|s)$ be a policy over a vocabulary $A$ of size $V$. The entropy gradient can be decomposed into contributions from a head set $A_\delta^c$ and a tail set $A_\delta = \{a \in A \mid \pi_\theta(a|s) < \delta \}$, for any threshold $\delta \in (0, 1)$:
\begin{equation}
    \nabla_\theta H(\pi_\theta) = \underbrace{- \sum_{a \in A_\delta^c} \nabla_\theta \pi_\theta(a|s) (1 + \log \pi_\theta(a|s))}_{\mathbf{G}_{\text{head}}} \underbrace{- \sum_{a \in A_\delta} \nabla_\theta \pi_\theta(a|s) (1 + \log \pi_\theta(a|s))}_{\mathbf{G}_{\text{tail}}}.
\end{equation}
We are concerned exclusively with the scalar factor $(1 + \log \pi_\theta(a|s))$, which is contributed by the entropy term alone. The per-token policy Jacobian $\nabla_\theta \pi_\theta(a|s)$ lies outside the quantity below. The cumulative magnitude of these logarithmic scaling factors over the tail, which we define as the Total Tail Instability (TTI), is lower-bounded and grows linearly with $V$. Specifically:
\begin{equation}
    \text{TTI} := \sum_{a \in A_\delta} |1 + \log \pi_\theta(a|s)| \ge \left(V - \frac{1}{\delta}\right) |1 + \log \delta|.
\end{equation}
This bound holds provided $V > 1/\delta$, a condition easily met in LLMs.
\end{theorem}

\begin{proof}

\textbf{1. Bounding the Magnitude of Each Term.}
For any token $a \in A_\delta$, by definition, $0 < \pi_\theta(a|s) < \delta$. Since the logarithm is a monotonically increasing function, this implies $\log \pi_\theta(a|s) < \log \delta$. Therefore, for each term in the TTI sum, we can establish a lower bound on its magnitude:
\begin{equation}
    |1 + \log \pi_\theta(a|s)| > |1 + \log \delta|.
\end{equation}
This holds true because for any small $\delta < 1/e \approx 0.36$, the term $(1 + \log \delta)$ is negative, and its magnitude increases as $\delta$ approaches zero.

\textbf{2. Bounding the Number of Terms.}
From Theorem~\ref{thm:tail_dominance}, we have a tight lower bound on the size of the tail set $|A_\delta|$:
\begin{equation}
    |A_\delta| \ge V - \frac{1}{\delta}.
\end{equation}
By combining these two results, we can lower-bound the Total Tail Instability (TTI):

\begin{align}
\text{TTI} = \sum_{a \in A_\delta} |1 + \log \pi_\theta(a|s)| &> \sum_{a \in A_\delta} |1 + \log \delta| \\
&= |A_\delta| \cdot |1 + \log \delta| \\
&\ge \left(V - \frac{1}{\delta}\right) |1 + \log \delta|.
\end{align}
This concludes the proof. The TTI, which aggregates the entropy-specific amplification factor over the tail, has a magnitude that scales at least linearly with the vocabulary size $V$.
\end{proof}

\textbf{Scope of Theorem~\ref{thm:gradient_instability}.}
TTI is a \textbf{probability-space tail-sensitivity diagnostic for the entropy objective}. The omitted Jacobian $\nabla_\theta \pi_\theta(a|s)$ is generally small for tail tokens, and the realized parameter update further depends on the optimizer, on gradient clipping and on numerical stabilization, so a complete characterization of that update would have to model all of these jointly.

What Theorem~\ref{thm:gradient_instability} establishes is the following. The entropy objective, alone among the terms in an RLVR objective, attaches a factor whose magnitude is at least $|1+\log\delta|$ to every one of the at least $V - 1/\delta$ tokens in the tail. The \textit{objective} therefore allocates a share of its optimization pressure to a token set of size $\Theta(V)$ that carries a vanishing share of the probability mass, and this allocation follows from the objective's form rather than from any particular optimizer. This is the structural asymmetry Info-GRPO exploits: the conditional-entropy branch of Eq.~(\ref{eq:info_objective}) drives probabilities toward the well-behaved end of the logarithm, while the marginal branch, which does not, is down-weighted by $\alpha < 1$.

\paragraph{Additional remarks on Proposition~\ref{thm:main_singularity}.}
The following remarks continue Remarks 1--3 of the main text.

\textit{Remark 4.} \textit{A large $V$ makes a dominant tail set inevitable}.
The first claim formalizes that a massive vocabulary forces an extremely sparse distribution, with the number of tail tokens growing linearly in $V$. In LLMs the region where the $\log \pi_\theta(a|s)$ factor is large is therefore not a fringe case but nearly the entire action space, as Figure~\ref{statistic_pic}(b) illustrates on the sampling distribution of Qwen2.5-7B.

\textit{Remark 5.} \textit{A dominant tail set spreads optimization pressure over the vocabulary}.
The second claim shows the cumulative amplification potential from tail tokens grows linearly in $V$. The entropy update therefore aggregates tens of thousands of terms whose scalar factors are large in magnitude, so a growing share of the entropy signal is allocated to tokens carrying almost no probability mass rather than to the few head tokens that determine the reasoning path. This is a statement about how the objective allocates pressure, not a claim that the resulting update necessarily diverges.

\textit{Remark 6.} \textit{The issue is at the objective level, complementary to optimizer-level safeguards}.
Gradient clipping and numerical stabilization (\emph{e.g.}, log-space computation, $\epsilon$-flooring) are indispensable in LLM training, but they operate downstream of the objective: clipping bounds the norm of the aggregated update without changing its direction, and numerical stabilization prevents overflow without removing the objective-level pressure to spread probability mass across the tail. Info-GRPO reshapes the objective itself and is therefore complementary to these safeguards rather than a replacement for them. All experiments in this paper, for both baselines and Info-GRPO, use identical clipping and stabilization settings (Appendix~\ref{Details}), so the differences we report are not attributable to them.

\subsection{Related Work}\label{Related}

\textbf{Reinforcement Learning for LLM Reasoning.}
Reinforcement learning has emerged as a pivotal paradigm for improving the reasoning capabilities of Large Language Models (LLMs)~\citep{o12024,o32025,grok2024,qwq2024,guo2025deepseek}. Early work focused on aligning models with human preferences, typically using Reinforcement Learning from Human Feedback (RLHF)~\cite{ouyang2022training} on top of supervised fine-tuning and parameter-efficient adaptation~\citep{hu2021lora,zhang2026primacy,zhang2026hidden}.
This domain could be categorized into online and offline preference optimization.
Online methods~\citep{schulman2017proximal,williams1992simple,shao2024deepseekmath} generate responses dynamically during training, receiving real-time feedback. In contrast, offline methods~\citep{rafailov2023direct,meng2024simpo,ethayarajh2024kto} optimize policies using pre-collected preference datasets.
Traditional methods like Proximal Policy Optimization (PPO)~\cite{schulman2017proximal} and REINFORCE~\cite{williams1992simple} are computationally expensive and suffer from instability due to the large and discrete action space. A significant advancement is the development of value-model-free methods, such as Group Relative Policy Optimization (GRPO)~\cite{shao2024deepseekmath}. It addresses the instability in PPO by using trajectory-level comparisons instead of value networks, thereby reducing computational costs and enhancing the robustness of the training process.
Reinforcement Learning with Verifiable Rewards (RLVR)~\cite{lambert2024tulu,shao2025spurious} has also emerged as a promising alternative, demonstrating how outcome-based reward signals can enhance reasoning, particularly in domains demanding rigorous logical deduction like mathematics and programming.
A growing understanding in the literature~\cite{gandhi2025cognitive} indicates that the presence of reasoning behaviors, rather than merely correct answers, is a primary driver of performance gains in RLVR.
Recent works such as DRA-GRPO~\cite{chen2025dragrpo} aim to address this by explicitly incorporating semantic diversity into the reward computation.
S-GRPO~\cite{dai2025s} improves the performance by encouraging conciseness and incentivizing early thinking termination.
Dr. GRPO~\cite{liu2025understanding} removes the length and std normalization terms to avoid the optimization bias in GRPO.
Light-R1~\citep{wen2025light} introduces curriculum training with progressively increasing data difficulty combined with multi-stage post-training.
DAPO~\cite{yu2025dapo} proposes four effective techniques, such as clip-higher, dynamic sampling, token-level policy gradient loss, and overlong reward shaping.
VAPO~\cite{yue2025vapo} further integrates the value model by proposing length-adaptive GAE.
ORZ~\cite{hu2025open} also utilizes a value model for advantage estimation with the Monte Carlo estimation.
GRPO has also been extended to jointly train hierarchical multi-agent systems~\citep{hu2026himaecomenablingjointtraining,hu2026oxygent,zhou2026hierarchical}.
Despite the successes of these methods, challenges remain.
Notably, GRPO-based approaches struggle with insufficient exploration and the lack of diversity in generated solutions~\cite{chen2025dragrpo}.
In this work, we focus on extending GRPO to achieve improvements in exploration and exploitation capabilities.

\textbf{Entropy Regularization for Reinforcement Learning.}
Entropy regularization has become an essential technique in RL to address issues such as premature convergence and insufficient exploration~\cite{hong2018diversity}. Early works focused on entropy as a means to encourage exploration in environments with high uncertainty~\cite{mnih2015human,mnih2016asynchronous,haarnoja2017reinforcement,schulman2017equivalence,schulman2017proximal,haarnoja2018soft}.
In particular, the maximum entropy principle~\cite{ziebart2008maximum,toussaint2009robot} has been used to balance reward maximization with policy stochasticity. It has been extended to language model training, where entropy-based terms are introduced into the reward function to enhance the model's exploratory behaviors during reasoning tasks.
Recent works focus on forking tokens, which introduce new reasoning paths to improve the reasoning performance of LLMs when.
Wang et al.~\cite{wang2025beyond} highlight the importance of high-entropy tokens in driving reasoning behavior.
FR3E~\cite{zheng2025first} identifies high-uncertainty decision points in reasoning trajectories and builds intermediate feedback by conducting targeted rollouts.
Other recent advancements in entropy-based exploration strategies, such as diversity-driven exploration~\cite{chen2025dragrpo} and the RL with entropy-augmented advantage~\cite{cheng2025reasoning}, propose solutions by introducing entropy regularization into the advantage function. These methods reinforce exploratory behaviors, allowing LLMs to tackle complex reasoning tasks more effectively.
In addition, Skywork-OR1~\cite{he2025skywork} utilizes the appropriate entropy control to mitigate premature convergence and improve test outcomes.
1-shot RLVR~\cite{wang2025reinforcement} promotes diverse exploration in outputs by adding an entropy loss with a coefficient to enhance model performance.
KL-Cov and Clip-Cov~\citep{cui2025entropy} clip and apply the KL penalty to high-covariance tokens, respectively, to prevent entropy collapse.
While effective at both a macroscopic level (preventing overall policy collapse) and a microscopic level (guiding exploration at individual token choices), it is challenging for entropy-based methods to balance exploration and exploitation, since the uncertainty introduced by entropy to promote exploration may weaken the confidence of the model.
In this paper, we extend entropy-regularized GRPO with mutual information to address the contradiction between diversity and confidence.

\textbf{Mutual Information for Structured Exploration.}
In unsupervised learning, Mutual Information (MI) has been widely used to capture dependencies between random variables and improve the diversity of learned representations~\cite{hjelm2018learning,poole2019variational,wen2020mutual,rakelly2021mutual,chen2024learning}.
The power of MI in learning disentangled representations in unsupervised settings is evident in methods like InfoGAN~\cite{chen2016infogan} and InfoVAE~\cite{zhao2017infovae}.
InfoGAN has disentangled prominent attributes to show its capacity for unsupervised discovery of interpretable concepts.
Similarly, InfoVAE addresses limitations of variational autoencoders by incorporating an explicit mutual information constraint between the latent code and the generated data within its loss function.
MI also underpins self-supervised contrastive learning, a field that employs the InfoNCE loss~\cite{chen2020simple,zhang2023On,lee2024soft} to maximize the similarity between positive sample pairs and minimize the similarity between negative sample pairs by estimating mutual information.
In the context of LLMs, mutual information has also been explored for enhancing model reasoning capabilities by ensuring that reasoning steps are not overly deterministic or constrained. For instance, the submodular mutual information used in GRPO-based methods~\cite{chen2025dragrpo} aims to downweight redundant completions and focus on diverse reasoning outputs.
Qian et al.~\cite{qian2025demystifying} investigate the reasoning trajectory of large reasoning models from the perspective of information theory. They find that the MI between intermediate representations and the answer arrives at peaks corresponding to tokens that indicate reflection or transition.
Moreover, InfoPO~\cite{xiao2025infopo} optimizes the conditional mutual information
between responses and preferences given a prompt to avoid the Bradley-Terry assumption.
In the context of controllability and randomness in generative models, the mutual information approach is particularly useful for balancing control over the model's output while still allowing for enough randomness to explore diverse reasoning strategies.
This motivates the emergence of Info-GRPO in this paper, which explores the combination of unsupervised mutual information maximization and RL techniques.
Through maximizing the mutual information conditioned on a new latent variable, our method improves the quality of reasoning and ensures that LLM can handle more complex and abstract reasoning tasks.

\begin{table}[t]
\centering
\small
\caption{
Training hyperparameters for Info-GRPO and baseline methods.
Unless otherwise noted, models are trained by us using the same infrastructure and default settings.
DAPO, Clip-Cov, and KL-Cov follow the official codebases and configurations.
}
\label{tab_training_details}
\setlength{\tabcolsep}{4pt}
\begin{tabular}{lcccc}
\toprule
\textbf{Hyperparameter}
& \textbf{Info-GRPO}
& \textbf{DAPO}
& \textbf{Clip-Cov}
& \textbf{KL-Cov} \\
\midrule
Training batch size                & 128  & 128  & 256  & 256  \\
Responses per prompt ($G$)         & 16   & 16   & 8    & 8    \\
Max prompt length                  & 2048 & 2048 & 2048 & 2048 \\
Max response length                & 8192 & 8192 & 8192 & 8192 \\
Learning rate ($\eta$)                      & $1\times10^{-6}$ & $1\times10^{-6}$ & $5\times10^{-7}$ & $5\times10^{-7}$ \\
Sampling temperature               & 1.0   & 1.0 & 1.0 & 1.0 \\
\midrule
Clipping ratio (low / high)        & 0.2 / 0.2 & 0.2 / 0.28 & 1.0 / 1.0 & 0.2 / 0.2 \\
Covariance clip ratio              & --    & -- & $2\times10^{-4}$ & $2\times10^{-3}$ \\
Covariance bounds                  & --    & -- & 1.0 / 5.0 & -- \\
\midrule
Enable filter groups & No & True & True  & True  \\
Overlong buffer enabled            & No    & Yes & No  & No  \\
Overlong buffer length             & --    & 4096 & 2048 & 2048 \\
Overlong penalty factor            & --    & 1.0 & 1.0 & 1.0 \\
\bottomrule
\end{tabular}
\end{table}

\subsection{{Full Experimental Settings}}\label{Details}

\begin{algorithm}[t]
\caption{The Info-GRPO Training Algorithm}
\label{alg_info_grpo}
\begin{algorithmic}[1]
\STATE \textbf{Input:} Initial policy $\pi_{\boldsymbol{\theta}}$, dataset $\mathcal{D}$, learning rate $\eta$.
\STATE \textbf{Initialize:} Policy parameters $\boldsymbol{\theta}$.
\FOR{each training iteration}
    \STATE Sample initial states $\{s_0^{(i)}\} \subset \mathcal{D}$ and latents $\{z^{(i)}\} \sim p(Z)$.
    \STATE Sample trajectories: $\mathcal{T}_{\text{ori}} = \{\pi_{\boldsymbol{\theta}}(\cdot \mid s_0^{(i)})\}$ and $\mathcal{T}_{\text{aug}} = \{\pi_{\boldsymbol{\theta}}(\cdot \mid s_0^{(i)},z^{(i)})\}$.
    \STATE Calculate $J_{\text{Info-GRPO}}(\boldsymbol{\theta})$ per Eq.~(\ref{eq:info-grpo_objective}).
    \STATE Update $\boldsymbol{\theta}$: $\boldsymbol{\theta} \leftarrow \boldsymbol{\theta} + \eta \nabla_{\boldsymbol{\theta}} J_{\text{Info-GRPO}}$.
\ENDFOR
\STATE \textbf{Return:} Optimized parameters $\boldsymbol{\theta}$.
\end{algorithmic}
\end{algorithm}

Below, we introduce more training hyperparameters across different methods.
In each rollout step, $G=16$ responses are sampled per prompt with a temperature of 1.0.
All models are trained on a single node with 8 NVIDIA B200 GPUs.
In addition to our proposed Info-GRPO, we include Dr. GRPO~\citep{liu2025understanding}, DARS~\citep{yang2025depth}, Light-R1~\citep{wen2025light}, Skywork-OR1~\citep{he2025skywork},  DAPO~\citep{yu2025dapo}, KL-Cov and Clip-Cov~\citep{cui2025entropy}, as well as DAPO~\citep{yu2025dapo}.
Unless otherwise specified, all models are trained by us using the same training pipeline and core hyperparameter settings as Info-GRPO, ensuring a fair comparison.
In particular, gradient clipping, loss scaling and the numerical stabilization of log-probabilities are configured identically for every method, including the entropy-regularized baselines. These safeguards therefore act identically across all methods, and the differences we report originate at the objective level.
In Table~\ref{multi_task_result}, every method except those marked ${\dag}$ (Dr. GRPO, DARS-7B-ET, DARS-7B-HW and Light-R1-7B-DS) was trained from scratch by us under this unified pipeline on identical hardware. The ${\dag}$ methods release checkpoints and are evaluated as published.
The target entropy is set to 0.2 for Skywork-OR1~\citep{he2025skywork}, while the clipping parameters for DAPO~\citep{yu2025dapo} are $\epsilon_{low}=0.2$ and $\epsilon_{high}=0.28$.
For KL-Cov, Clip-Cov, and DAPO, we follow the official implementations and configurations released by the respective authors.
Moreover, Dr. GRPO~\citep{liu2025understanding}, DARS~\citep{yang2025depth}, and Light-R1~\citep{wen2025light} directly reuse the officially released checkpoints for evaluation and are therefore not retrained.
Table~\ref{tab_training_details} summarizes the key training hyperparameters for the methods.
While Info-GRPO adopts a unified setup with mutual information regularization, KL-Cov and Clip-Cov differ mainly in their covariance-based clipping strategies and batch configurations, and DAPO further introduces asymmetric clipping and overlong response buffering.
All results are obtained under the best-performing official or default settings.

\begin{table}[!t]
\centering
\small
\caption{
Comparison of methods on multiple backbones and benchmarks ($Pass@1$, \%). `${\dag}$' indicates using the official model.
}
\label{multi_task_result_p1}
\setlength{\tabcolsep}{3.5pt}
\begin{tabular}{llcccccc}
\toprule
Backbone & Method & AIME24 & AIME25 & AMC & MATH500 & Minerva & \textbf{Average} \\
\midrule
\multirow{7}{*}{\textit{\textbf{Qwen2.5-7B}}}
& Base Model & 13.33  & 6.67  & 40.96  & 75.40  & 35.66  & 34.40  \\
& GRPO & \underline{16.67}  & \underline{13.33}  & \underline{57.83}  & 77.60  & 37.50  & 40.59  \\
& DAPO & 13.33  & \textbf{16.67}  & 51.81  & 78.60  & 35.66  & 39.21  \\
& CLIP-Cov & \textbf{20.00}  & \underline{13.33}  & 51.81  & \textbf{79.60}  & \underline{38.24}  & \underline{40.60}  \\
& KL-Cov & \underline{16.67}  & 10.00  & \underline{57.83 } & 74.20  & 34.56  & 38.65  \\
& Skywork-OR1 & \underline{16.67}  & 10.00  & \underline{57.83}  & 77.80  & \underline{38.24}  & 40.11  \\
& Info-GRPO (Ours) & \underline{16.67}  & \underline{13.33}  & \textbf{59.04}  & \underline{79.40}  & \textbf{40.07}  & \textbf{41.70}  \\
\midrule
\multirow{6}{*}{\makecell[l]{\textit{\textbf{Qwen2.5-}}\\\textit{\textbf{Math-7B}}}}
& Base Model & 13.33  & 6.67  & 49.40  & 82.20  & 35.29  & 37.38  \\
& GRPO & 23.33  & 10.00  & 69.88  & \textbf{86.80}  & \underline{37.13}  & 45.43  \\
& DAPO & \underline{26.67}  & 6.67  & 77.11  & \underline{86.60}  & 36.03  & 46.62  \\
& Dr. GRPO ${\dag}$  & \textbf{30.00}  & \underline{13.33}  & 55.42  & 79.60  & 34.19  & 42.51  \\
& Skywork-OR1 & \underline{26.67}  & \textbf{16.67}  & \textbf{80.72}  & 84.20  & \underline{37.13}  & \underline{49.08} \\
& Info-GRPO (Ours) & \textbf{30.00}  & \textbf{16.67}  & \underline{78.31}  & \textbf{86.80}  & \textbf{38.97}  & \textbf{50.15}  \\
\midrule
\multirow{6}{*}{\makecell[l]{\textit{\textbf{DeepSeek-R1}}\\\textit{\textbf{-Distill-Qwen}}\\\textit{\textbf{-7B}}}}
& Base Model & 53.33  & 33.33  & 79.52  & 92.22  & 39.34  & 59.55  \\
& GRPO & \textbf{66.67}  & 40.00  & 79.52  & \underline{95.00}  & 42.65  & 64.77  \\
& DAPO & \underline{63.33}  & 36.67  & 80.72  & 93.00  & 44.12  & 63.57  \\
& Light-R1-7B-DS ${\dag}$ & \underline{63.33}  & \underline{43.33}  & \underline{81.93}  & 93.20  & \underline{45.59}  & \underline{65.48}  \\
& Skywork-OR1 & \underline{63.33}  & 40.00  & \underline{81.93}  & \textbf{95.20}  & 42.28  & 64.55  \\
& Info-GRPO (Ours) & 56.67  & \textbf{50.00}  & \textbf{84.34}  & 94.80  & \textbf{46.32}  & \textbf{66.43}  \\
\midrule
\multirow{4}{*}{\makecell[l]{\textit{\textbf{{DeepSeek-R1}}}\\\textit{\textbf{{-Distill-Llama}}}\\\textit{\textbf{{-8B}}}}}
& Base Model & 46.67  & 26.67  & 77.11  & 85.40  & 36.76  & 54.52  \\
& GRPO & 46.67  & 36.67  & 80.72  & 84.40  & \textbf{38.97}  & 57.49  \\
& DAPO & \underline{56.67}  & \underline{40.00}  & \underline{83.13}  & \textbf{92.00}  & 37.87  & \underline{61.93}  \\
& Info-GRPO (Ours) & \textbf{60.00}  & \textbf{46.67}  & \textbf{84.34}  & \underline{91.60}  & \underline{38.60}  & \textbf{64.24} \\
\bottomrule
\end{tabular}
\end{table}

\begin{table}[!t]
\centering
\small
\caption{
Comparison of methods on multiple backbones and benchmarks ($Pass@8$, \%). `${\dag}$' indicates using the official model.
}
\label{multi_task_result_p8}
\setlength{\tabcolsep}{3.5pt}
\begin{tabular}{llcccccc}
\toprule
Backbone & Method & AIME24 & AIME25 & AMC & MATH500 & Minerva & \textbf{Average} \\
\midrule
\multirow{7}{*}{\textit{\textbf{Qwen2.5-7B}}}
& Base Model & 30.00  & 20.00  & 67.47  & 88.80  & 47.79  & 50.81  \\
& GRPO & 30.00  & \underline{23.33}  & 73.49  & 88.80  & 47.79  & 52.68  \\
& DAPO & \underline{33.33}  & \underline{23.33}  & 67.47  & \textbf{90.60}  & \underline{50.74}  & \underline{53.09}  \\
& CLIP-Cov & 30.00  & 20.00  & 71.08  & \underline{90.20}  & \textbf{51.10}  & 52.48  \\
& KL-Cov & 30.00  & \underline{23.33}  & \underline{74.70}  & 83.60  & 48.53  & 52.03  \\
& Skywork-OR1 & \underline{33.33}  & 20.00  & \underline{74.70}  & 87.60  & 48.90  & 52.91  \\
& Info-GRPO (Ours) & \textbf{36.67}  & \textbf{26.67}  & \textbf{77.11}  & 89.00  & 48.90  & \textbf{55.67}  \\
\midrule
\multirow{6}{*}{\makecell[l]{\textit{\textbf{Qwen2.5-}}\\\textit{\textbf{Math-7B}}}}
& Base Model & 26.67  & 20.00  & 75.90  & \underline{92.80}  & 45.22  & 52.12  \\
& GRPO & \underline{43.33}  & \underline{23.33}  & 79.52  & 92.40  & 47.43  & \underline{57.20}  \\
& DAPO & 30.00  & \textbf{26.67}  & \underline{86.75}  & \textbf{93.40}  & 48.16  & 57.00  \\
& Dr. GRPO ${\dag}$ & \textbf{46.67}  & \underline{23.33}  & 74.70  & 88.40  & 48.16  & 56.25  \\
& Skywork-OR1 & 33.33  & \underline{23.33}  & \textbf{87.95}  & 90.40  & \underline{48.90}  & 56.78  \\
& Info-GRPO (Ours) & 36.67  & \textbf{26.67}  & 84.34  & \textbf{93.40}  & \textbf{49.26}  & \textbf{58.07}  \\
\midrule
\multirow{6}{*}{\makecell[l]{\textit{\textbf{DeepSeek-R1}}\\\textit{\textbf{-Distill-Qwen}}\\\textit{\textbf{-7B}}}}
& Base Model & 80.00  & 56.67  & 91.57  & 98.00  & 53.68  & 75.98  \\
& GRPO & \underline{83.33}  & 60.00  & \textbf{95.18} & 98.20  & 56.62  & 78.67  \\
& DAPO & \underline{83.33}  & 63.33  & 92.77  & 98.20  & \underline{56.99}  & 78.92  \\
& Light-R1-7B-DS ${\dag}$ & \underline{83.33}  & \underline{66.67}  & 90.36  & \textbf{99.20}  & 55.88  & \underline{79.09}  \\
& Skywork-OR1 & 80.00  & 63.33  & \underline{93.98}  & \underline{99.00}  & 56.25  & 78.51  \\
& Info-GRPO (Ours) & \textbf{86.67}  & \textbf{70.00}  & \textbf{95.18}  & 98.40  & \textbf{57.35}  & \textbf{81.52}  \\
\midrule
\multirow{4}{*}{\makecell[l]{\textit{\textbf{{DeepSeek-R1}}}\\\textit{\textbf{{-Distill-Llama}}}\\\textit{\textbf{{-8B}}}}}
& Base Model & 80.00  & 53.33  & 92.77  & 97.20  & 53.31  & 75.32  \\
& GRPO & \underline{83.33}  & \underline{56.67}  & \underline{95.18}  & \underline{97.60}  & \underline{54.41}  & 77.44  \\
& DAPO & \underline{83.33}  & \textbf{60.00}  & 93.98  & 97.40  & 54.04  & \underline{77.75}  \\
& Info-GRPO (Ours) & \textbf{86.67}  & \textbf{60.00}  & \textbf{96.39}  & \textbf{98.20}  & \textbf{56.99} & \textbf{79.65} \\
\bottomrule
\end{tabular}
\end{table}

\subsection{{Comparative Experiments}}\label{Compare}

More comparisons and analyses are provided as follows:

\begin{table}[t]
\centering
\small
\caption{Comparison with ETR~\citep{etr2026} on DeepSeek-R1-Distill-Qwen-7B ($Pass@1$, \%), under an identical decoding protocol.}
\label{tab_etr}
\begin{tabular}{lcccc}
\toprule
Method & AIME24 & AMC23 & MATH500 & \textbf{Avg.} \\
\midrule
ETR & \textbf{56.7} & \textbf{87.5} & 90.6 & 78.3 \\
Info-GRPO (Ours) & \textbf{56.7} & 84.3 & \textbf{94.8} & \textbf{78.6} \\
\bottomrule
\end{tabular}
\end{table}

\paragraph{Comparison with trajectory-shaping entropy methods.}
ETR~\citep{etr2026} shapes the \emph{trend} of uncertainty over a chain of thought rather than its level, rewarding a monotonically decreasing trajectory. Info-GRPO exhibits the same downward trend in Figure~\ref{reward_curve}(c--f) without an explicit trajectory penalty, since the conditional-entropy term supplies that pressure implicitly. Table~\ref{tab_etr} compares the two under an identical decoding protocol: they tie on AIME24, ETR is stronger on AMC23, and Info-GRPO is stronger on the much larger MATH500 ($94.8$ \emph{vs.} $90.6$), giving a comparable overall average. The two mechanisms therefore appear complementary.

\paragraph{Comparison on $Pass@1$ and $Pass@8$.}
Table~\ref{multi_task_result_p1} and Table~\ref{multi_task_result_p8} further report $Pass@1$ and $Pass@8$ results across multiple benchmarks and backbones to complement the main $Avg@8$ evaluation.
It could be seen that Info-GRPO consistently achieves the best average performance based on all base models.
For instance, for $Pass@1$, it improves over GRPO by 4.72\% and 6.75\% on average based on Qwen2.5-Math-7B and DeepSeek-R1-Distill-Llama-8B, respectively.
On AIME24, the results show weaker performance on strong distilled backbones such as DeepSeek-R1-Distill-Qwen-7B, where single-attempt success is already near saturation.
For $Pass@8$, it improves over GRPO by 2.99\% and 2.85\% on average based on Qwen2.5-7B and DeepSeek-R1-Distill-Qwen-7B, respectively. It could be observed that Info-GRPO performs competitively across most settings, with minor exceptions.
Overall, these results suggest that Info-GRPO primarily enhances the quality and stability of the reasoning distribution under multi-sample evaluation, yielding consistent gains across models and metrics, while its advantage on single-attempt success naturally diminishes when the base model performance is already saturated.

\paragraph{Comparison on $MV@N$.}
We also evaluate test-time compute scaling via Majority Voting. We conducted experiments on the AIME24 benchmark using both the Qwen2.5-Math-7B and DeepSeek-R1-Distill-Qwen-7B backbones. We evaluated the vanilla GRPO baseline and our Info-GRPO model across increasing compute budgets ($N \in \{4, 8, 16, 32, 64, 128\}$ samples per question) and applied Majority Voting (MV@N), reported in Table~\ref{mvn_compare} of the main text.
Info-GRPO exhibits significantly higher overall computational efficiency, reaching consensus and peak performance much earlier than the baseline.
On Qwen2.5-Math-7B, it hits peak accuracy (30.0) at N=8, while GRPO does not reach that level even at N=128, though on a 30-problem benchmark that gap is one problem.
On DeepSeek-R1-Distill-Qwen-7B, Info-GRPO hits 80.0 at N=8 \textit{vs.} GRPO at N=64.
This is consistent with Info-GRPO learning a diverse yet individually confident set of seed-conditioned reasoning modes during training, which reduces the test-time compute required to reach consensus on the correct answer.

\paragraph{Clarification of Comparison Methods using Official Configurations or Models.}
In Table~\ref{multi_task_result}, we further compare Info-GRPO with several recent methods using officially released codes or models under identical evaluation settings.

\begin{itemize}[leftmargin=*]
\item \textit{Qwen2.5-7B setting.} For KL-Cov and Clip-Cov~\citep{cui2025entropy}, which release code but not pretrained models, we train them using the official implementations and report their best converged performance.
Info-GRPO achieves superior results on most benchmarks, demonstrating consistent gains beyond entropy-based baselines.

\item \textit{Qwen2.5-Math-7B setting.} We compare with Dr. GRPO~\citep{liu2025understanding} by using its official model.
Although Dr. GRPO achieves the highest scores on AIME24, it performs substantially worse on other datasets, resulting in a much lower overall average compared to Info-GRPO, and this phenomenon is consistent across all three metrics.

\item \textit{DeepSeek-R1-Distill-Qwen-7B setting.} Light-R1~\citep{wen2025light} with the publicly available model weight improves over GRPO and DAPO, but Info-GRPO still achieves the best mean performance, indicating stronger overall robustness.

\end{itemize}

Overall, these comparisons highlight that while some recent methods excel on specific benchmarks, Info-GRPO delivers more stable and consistent improvements across tasks and model families.

\paragraph{Comparisons on More Model Series.}
To further validate scalability, we evaluate GRPO, DAPO, and Info-GRPO on DeepSeek-R1-Distill-Llama-8B under identical training settings.
As shown in Table~\ref{multi_task_result}, Table~\ref{multi_task_result_p1}, and Table~\ref{multi_task_result_p8}, Info-GRPO consistently achieves the best performance across multiple benchmarks and metrics, confirming that its benefits generalize beyond the Qwen family to other backbone architectures.

\begin{table}[t]
\centering
\small
\caption{Performance and diversity comparison on AIME25 using the Qwen2.5-7B base model.}
\label{failure_result}
\begin{tabular}{lcc}
\toprule
Method & $Avg@8$ & Diversity (Euclidean distance)\\
\midrule
GRPO & 10.83 & 0.206 \\
DAPO & 11.25 & 0.190 \\
Clip-Cov & 12.08 & 0.217 \\
KL-Cov & 11.25 & 0.205 \\
Skywork-OR1 & 7.50  & 0.239\\
Info-GRPO (Ours) & 8.75 & 0.287  \\
\bottomrule
\end{tabular}
\vspace{-3mm}
\end{table}

\paragraph{Failure Analysis.}
For the underperformance of our method on AIME25 in Table~\ref{multi_task_result}, we investigate the underlying causes as follows.

\textit{Deeper analysis.}
AIME25 is a highly demanding mathematical reasoning benchmark, where the headroom for diversity-based improvements is fundamentally constrained by the reasoning capacity of the base model. Although Entropy-regularized methods such as Skywork-OR1 and Info-GRPO generally benefit from increased sample diversity, the effectiveness of diversity depends critically on the quality of the candidate reasoning trajectories that the base model can generate.
\textit{When the underlying model’s capability is relatively limited, promoting diversity tends to increase the likelihood of generating low-quality or unstable reasoning paths, which can reduce the probability that any of the eight sampled responses is correct.}
This hypothesis is further supported by the performance of another recent entropy-regularized method, Skywork-OR1. Although it improves diversity, its performance (7.5) is even lower than Info-GRPO and all other baselines. This serves as an additional indicator that diversity alone does not guarantee improved reasoning accuracy when base-model ability is the limiting factor.

\textit{Empirical evidence.}
Prior works typically measure diversity using token-level entropy, which cannot capture the semantic similarity between whole reasoning trajectories. Following the practice in DRA-GRPO~\citep{chen2025dragrpo}, we adopt an embedding-based sequence-level diversity with the average Euclidean distance from pairwise embeddings.
Table~\ref{failure_result} illustrates that Info-GRPO achieves the highest diversity ceiling when averaging over all latent seeds. However, higher diversity does not correlate with higher accuracy on AIME25. Even at its diversity peak, Info-GRPO reaches 8.75, which is higher than Skywork-OR1 but remains below GRPO, DAPO, and Clip-Cov. This pattern is fully consistent with our hypothesis: \textit{When the base model struggles to produce stable reasoning trajectories, adding diversity amplifies variance more than it improves correctness.}

\subsection{{Ablation Study}}\label{Ablation}

The following results extend the ablation study of the main text.

\begin{table}[t]
\centering
\small
\caption{Cross-benchmark ablation on $\lambda$ using DeepSeek-R1-Distill-Qwen-7B ($Avg@8$, \%), on benchmarks disjoint from the headline AIME sets.}
\label{ablation_coef_cross}
\begin{tabular}{lcc}
\toprule
Coef $\lambda$ & AMC & Minerva \\
\midrule
0.5 & 79.52 & 39.71 \\
0.05 & 80.72 & 41.91 \\
0.005 (Used) & \textbf{83.13} & \textbf{44.94} \\
\bottomrule
\end{tabular}
\end{table}

\begin{table}[t]
\centering
\caption{
Ablation on coefficient $\lambda$ using DeepSeek-R1-Distill-Llama-8B (\%).
}
\label{ablation_coef_llama}
\begin{tabular}{ccccccc}
\toprule
\multirow{2}{*}{Coef $\lambda$} & \multicolumn{3}{c}{AIME24} & \multicolumn{3}{c}{AIME25} \\
\cmidrule{2-7}
 & $Pass@1$ & $Avg@8$ & $Pass@8$ & $Pass@1$ & $Avg@8$ & $Pass@8$ \\
\midrule
0.5 & 46.67 & 52.92 & 76.67 & 30.00 & 33.75 & 53.33\\
0.05 & 53.33 & 55.83 & 83.33 & 33.33 & 34.58 & 53.33\\
0.005 & \textbf{60.00} & \textbf{57.92} & \textbf{86.67} & \textbf{46.67} & \textbf{37.50} & \textbf{60.00}\\
\bottomrule
\end{tabular}
\end{table}

\paragraph{Ablation on Coefficient $\lambda$.}
Table~\ref{ablation_coef} of the main text shows a clear advantage for smaller values of $\lambda$ on DeepSeek-R1-Distill-Qwen-7B: lower coefficients such as $0.005$ and $0.002$ constrain exploration and consistently enhance generation quality and stability, which is why $0.005$ is selected as the default setting. This choice is not made on the headline test sets alone: the same ordering of $\lambda$ holds on AMC and Minerva, which are disjoint from AIME24/25 (Table~\ref{ablation_coef_cross}), and on a different architecture (Table~\ref{ablation_coef_llama}), whose trend closely aligns with the Qwen-based dynamics and also peaks at $0.005$. A single value of $\lambda$ is then shared by all four backbones and all six benchmarks, with no per-benchmark tuning.

\paragraph{Ablation on the Trace Budget.}
Table~\ref{ablation_length} of the main text varies the maximum response length on DeepSeek-R1-Distill-Qwen-7B. Increasing the budget from 2K to 8K tokens leads to consistent gains on both AIME benchmarks, indicating that a more generous length is important for complex reasoning tasks, as it provides the model with sufficient capacity to elaborate on logical steps and computations, while a constrained length can truncate complete reasoning chains.
The maximum response length is a backbone-dependent architectural constraint rather than a tuned hyperparameter: it is fixed at 4,096 tokens for the Qwen2.5 series to match their official configurations and at 8,192 tokens for the DeepSeek-R1-Distill series, whose reasoning traces are substantially longer, and the same limit is applied identically to every method trained on a given backbone. Pilot scaling runs beyond 8K yielded saturated performance at sharply higher inference cost, which is why 8K is used for the distilled models.

\begin{table}[t]
\centering
\small
\caption{
Ablation on the number of latent seeds (\%).
}
\label{ablation_seed}
\setlength{\tabcolsep}{4pt}
\begin{tabular}{lcccc}
\toprule
Number of Latent Seeds & $Pass@1$ & $Avg@8$ & $Pass@8$ & Diversity (Euclidean distance) \\
\midrule
3 & 30.00 &	25.00 &	40.00  & 0.28  \\
30 & 26.67 & 23.33 & 30.00  & 0.31 \\
10 (Used) & 30.00 &	25.00 &	36.67 & 0.30 \\
\bottomrule
\end{tabular}
\end{table}

\paragraph{Ablation on Number of Latent Seeds.}
Table~\ref{ablation_seed} shows a quantitative comparison of different numbers of latent seeds, complementing the trends observed in Figure~\ref{reward_curve}(a). Based on 3, 10, and 30 seeds, Info-GRPO achieves comparable performance on all metrics, with no substantial or consistent advantage from increasing the seed count. In particular, 10 seeds match the best $Pass@1$ and $Avg@8$ scores while maintaining competitive $Pass@8$ performance. The measured semantic diversity also varies only slightly across settings, indicating that increasing the number of seeds mainly expands the implicit exploration space rather than yielding proportional gains in performance or diversity. These results support the conclusion that performance gains tend to stabilize around 10 seeds, which provides a balanced trade-off between expressive exploration and stable optimization. Therefore, 10 latent seeds are adopted as the default configuration in our experiments.

\begin{figure}[t]
\centering
\includegraphics[width=1.0\textwidth]{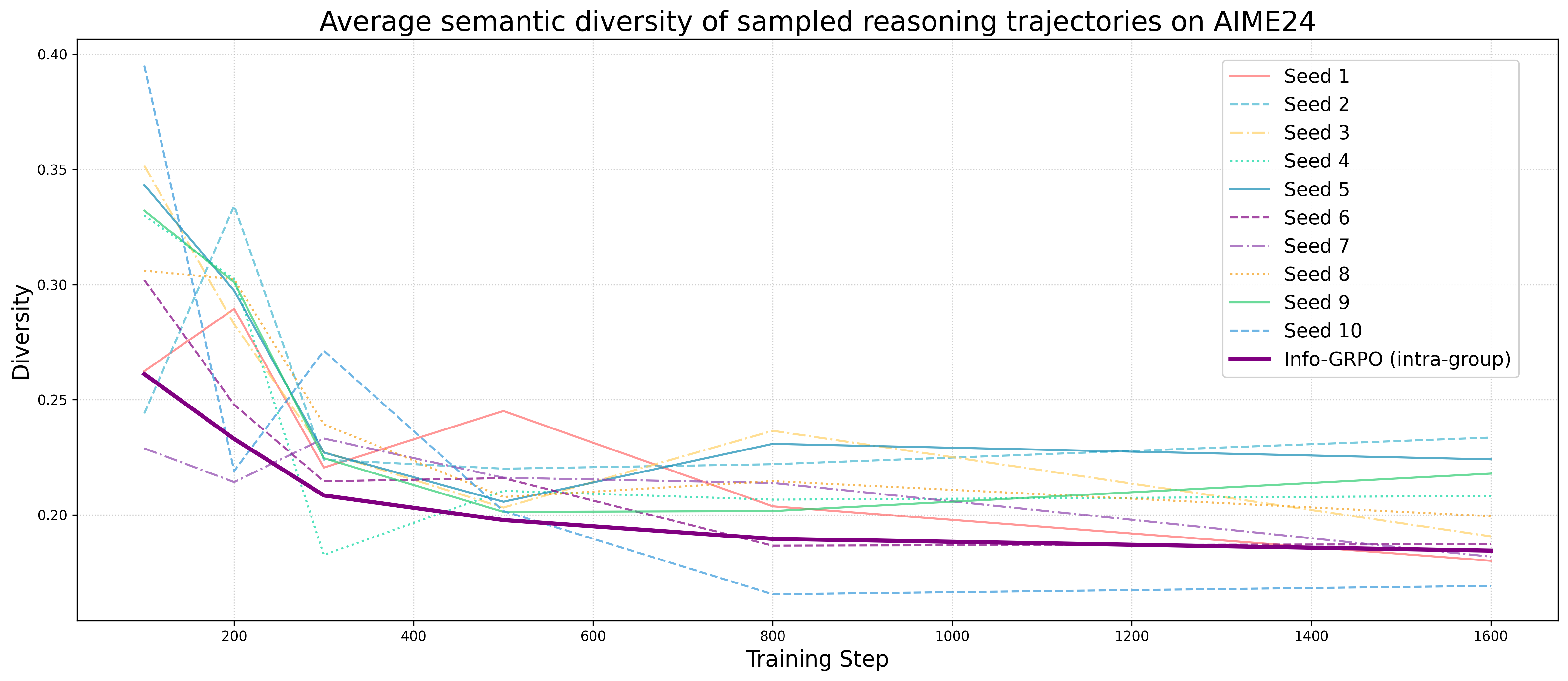}
\caption{
    Per-seed intra-group diversity dynamics of Info-GRPO (Qwen2.5-Math-7B) on AIME24 over an extended training horizon (up to 1600 steps).
    Each thin curve represents the intra-group diversity for a single latent seed, while the thick curve shows their average, reported as Info-GRPO (Intra-group) in Figure~\ref{reward_curve}(b).
    Although some seeds exhibit mild early fluctuations, all curves follow a generally downward trend and gradually stabilize, confirming that the averaged curve accurately reflects the typical training behavior across seeds.
}
\label{seed_curve}
\end{figure}

\paragraph{Analysis of the Policy Diversity.}
To complement the averaged intra-group diversity curve shown in Figure~\ref{reward_curve}(b), Figure~\ref{seed_curve} further visualizes the per-seed intra-group diversity dynamics of Info-GRPO across all 10 latent seeds used in our experiments.
Each curve corresponds to the diversity computed within a single latent seed, prior to averaging across seeds.
While a few seeds exhibit mild fluctuations during early training, the overall trend is consistently downward and stabilizes over the extended horizon.
This observation indicates that the averaged curve in Figure~\ref{reward_curve}(b) faithfully reflects a typical training behavior shared across seeds, rather than being dominated by a particular latent seed.
These results support the conclusion that Info-GRPO achieves stable and structured exploration by encouraging convergence within each latent-conditioned policy while preserving diversity across different seeds.

\begin{figure}[t]
\centering
\includegraphics[width=0.6\textwidth]{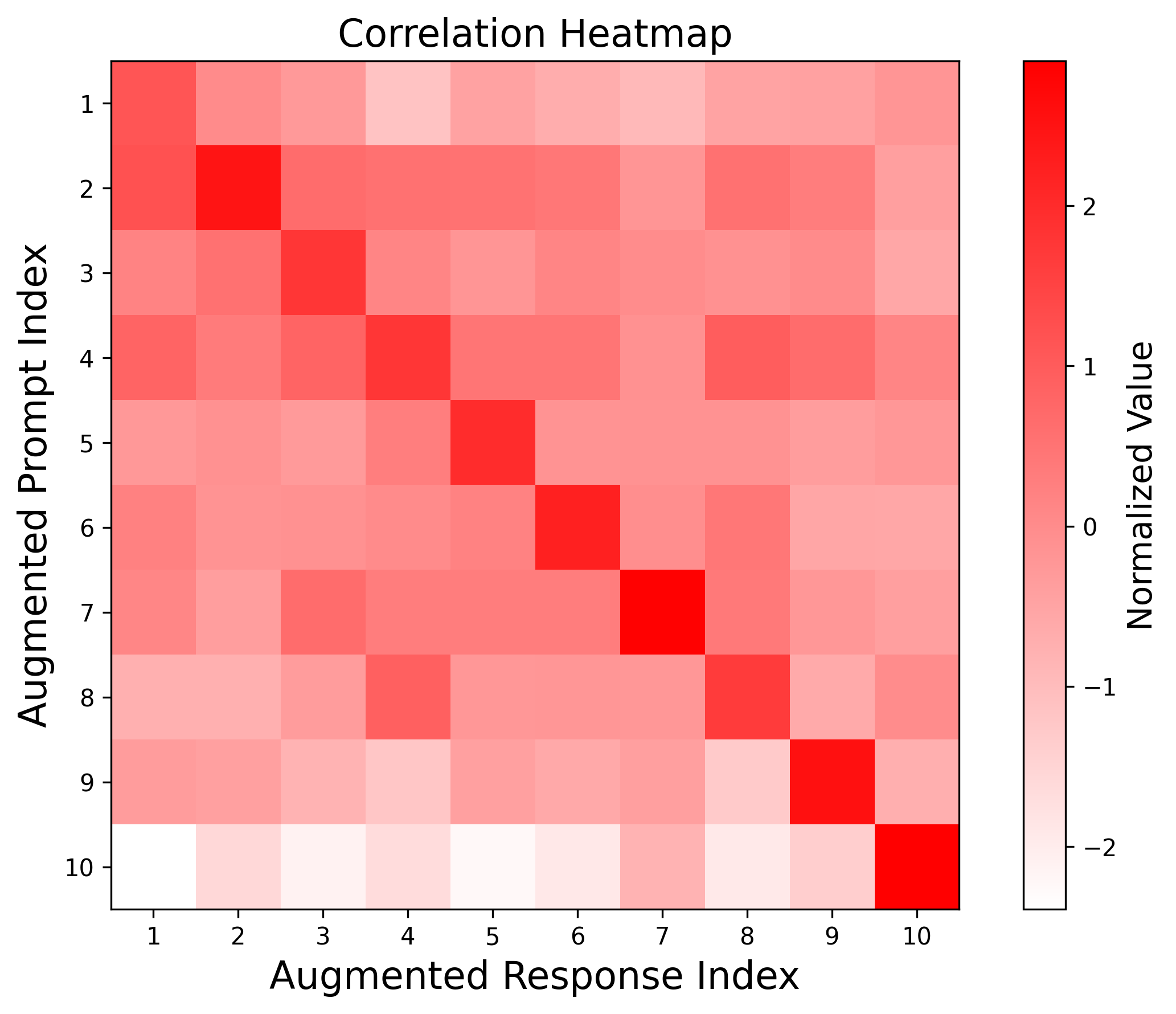}
\caption{The correlation heatmap for a case. It is generated by our Info-GRPO based on the DeepSeek-R1-Distill-Qwen-7B models.
}
\label{case_pic}
\end{figure}

\begin{table}[t]
\centering
\footnotesize
\caption{
The scores for 8 responses under different prompt seeds.
}
\label{case_result}
\setlength{\tabcolsep}{3pt}
\begin{tabular}{cccccccccccc}
\toprule
Response & No seed & Seed1 & Seed2 & Seed3 & Seed4 & Seed5 & Seed6 & Seed7 & Seed8 & Seed9 & Seed10\\
\midrule
R1 & 1 & 1 & 0 & 1 & 1 & 1 & 0 & 1 & 1 & 1 & 1 \\
R2 & 0 & 1 & 1 & 0 & 0 & 0 & 1 & 1 & 1 & 0 & 1 \\
R3 & 1 & 1 & 1 & 1 & 0 & 0 & 1 & 1 & 1 & 0 & 1 \\
R4 & 0 & 0 & 1 & 0 & 1 & 0 & 0 & 1 & 1 & 1 & 1 \\
R5 & 1 & 1 & 1 & 1 & 1 & 1 & 1 & 1 & 0 & 1 & 1 \\
R6 & 1 & 0 & 1 & 1 & 1 & 0 & 1 & 1 & 0 & 1 & 1 \\
R7 & 1 & 1 & 1 & 1 & 1 & 1 & 1 & 1 & 1 & 1 & 0 \\
R8 & 1 & 0 & 1 & 1 & 1 & 1 & 0 & 1 & 0 & 0 & 0 \\
\bottomrule
\end{tabular}
\end{table}

\subsection{Case Analysis}\label{Case}

\paragraph{Mathematical Samples.}

Note that in the main text, we only use augmented prompts during training, while evaluation is conducted without prompt augmentation.
Here, different augmented seeds are used at evaluation time solely for qualitative case analysis.
Figure~\ref{case_pic} presents a correlation heatmap for the following mathematical \textbf{Question 1}, and Table~\ref{case_result} reports the rewards of 8 sampled responses under different prompt seeds.
In this representative example, we observe that different seeds can lead the model to follow noticeably different solution strategies, which in turn result in different rewards.
Specifically, \textbf{Response 1.1} with reward 1 adopts a more systematic algebraic reasoning process, explicitly deriving variable constraints and exhaustively accounting for all valid configurations, leading to the correct solution, whereas \textbf{Response 1.2} with reward 0 follows a constructive, case-based reasoning path that prematurely restricts the solution space and consequently misses valid cases.
While this observation is based on a single illustrative example and does not imply that all seeds consistently induce such divergent reasoning behaviors, it nonetheless provides qualitative evidence that our correlation-guided approach is capable of steering the model toward different reasoning trajectories and encouraging exploration over alternative solution paths.

\begin{tcolorbox}[colframe=gray, colback=gray!10, title=Question 1 (Math), breakable]

Find the number of ways to place a digit in each cell of a 2x3 grid so that the sum of the two numbers formed by reading left to right is $999$, and the sum of the three numbers formed by reading top to bottom is $99$. The grid below is an example of such an arrangement because $8+991=999$ and $9+9+81=99$.
\[
\begin{array}{|c|c|c|}
\hline 0 & 0 & 8 \\
\hline 9 & 9 & 1 \\
\hline
\end{array}
\]

\end{tcolorbox}

\begin{tcolorbox}[
  colback=green!5,colframe=green!40!black,
  title=Response 1.1 (under seed 1 with reward 1),
  breakable
]

Alright, so I've got this problem where I need to figure out how many ways there are to place a digit in each cell of a 2$\times$3 grid. The conditions are that the sum of the two numbers formed by reading left to right in each row should be 999, and the sum of the three numbers formed by reading top to bottom in each column should be 99.

Let me try to visualize the grid. It's a 2$\times$3 grid, meaning there are 2 rows and 3 columns. Each cell will contain a digit from 0 to 9. The example given is:

\[
\begin{array}{|c|c|c|} \hline
0 & 0 & 8 \\ \hline
9 & 9 & 1 \\ \hline
\end{array}
\]

In this case, reading the first row left to right gives the number 008, which is 8, and the second row gives 991. Adding these together, $8 + 991 = 999$. For the columns, reading top to bottom, each column gives 09, 09, and 81. Adding these together, $9 + 9 + 81 = 99$.

Alright, so the key here is to figure out all such possible grids that satisfy these two conditions. Let me break it down.

First, let's denote the grid cells as follows:

\[
\begin{array}{|c|c|c|} \hline
a & b & c \\ \hline
d & e & f \\ \hline
\end{array}
\]

~

......

~

We find that:
\begin{itemize}
\item $c + f = 9$
\item $d = b + c + 1$
\item $a + d = 9$
\item $b + e = 9$
\end{itemize}

We express the variables in terms of $d$ and $c$:
\begin{itemize}
\item $a = 9 - d$
\item $b = d - c - 1$
\item $e = 10 - d + c$
\item $f = 9 - c$
\end{itemize}

We determine the constraints:
\begin{itemize}
\item $d$ ranges from 1 to 9
\item For each $d$, $c$ ranges from 0 to $d - 1$
\end{itemize}

The number of possible pairs $(c, d)$ is the sum from $d = 1$ to $9$ of the number of $c$ values, which is:
\[
\sum_{d=1}^{9} d = \frac{9 \times 10}{2} = 45
\]

Thus, the number of ways to place the digits in the grid is:
\[
\boxed{45}
\]

\end{tcolorbox}

\begin{tcolorbox}[
  colback=red!5,colframe=red!75!black,
  title=Response 1.2 (under seed 2 with reward 0),
  breakable
]

Alright, let me try to figure out how to solve this problem. So, we have a 2$\times$3 grid, right? We need to place digits in each cell such that two conditions are satisfied:

\begin{enumerate}
\item The sum of the two numbers formed by reading left to right in each row is 999.
\item The sum of the three numbers formed by reading top to bottom in each column is 99.
\end{enumerate}

~

......

~

\medskip
The key steps are:
\begin{enumerate}
\item The sum of the top digits $A + B + C = 9$.
\item The sum of the bottom digits $D + E + F = 9$.
\item The sum of the numbers formed by the columns must be 99.
\end{enumerate}

We analyze the columns:
\begin{itemize}
\item Column 1: $\text{col}_1 = D$ if $A = 0$, otherwise $10A + D$
\item Column 2: $\text{col}_2 = E$ if $B = 0$, otherwise $10B + E$
\item Column 3: $\text{col}_3 = F$ if $C = 0$, otherwise $10C + F$
\end{itemize}

The sum of these columns must be:
\[
\text{col}_1 + \text{col}_2 + \text{col}_3 = 99
\]

After thorough analysis, we find that the only valid grid is:
\begin{itemize}
\item $A = 0,\ B = 9,\ C = 0$
\item $D = 9,\ E = 0,\ F = 0$
\end{itemize}

This results in:
\begin{itemize}
\item $\text{col}_1 = 9$
\item $\text{col}_2 = 90$
\item $\text{col}_3 = 0$
\end{itemize}

Thus, the sum of the columns is $9 + 90 + 0 = 99$.

Therefore, the number of such grids is:
\[
\boxed{1}
\]

\end{tcolorbox}

\paragraph{Code Samples.}
We further analyse a code sample in \textbf{Question 2}.
The answers obtained from different seeds are provided in \textbf{Response 2.1} and \textbf{Response 2.2}.
It could be observed that although both solutions yield correct outputs, they rely on fundamentally different reasoning abstractions. One solution exploits problem-specific structural independence to collapse the state space, while the other solves the problem via a general-purpose dynamic programming formulation. This illustrates that different latent seeds can induce distinct yet valid reasoning trajectories rather than superficial variations. As a qualitative case, it is offered as an illustration of the seed-conditioned behavioural modes reported quantitatively in Figure~\ref{reward_curve}(b), not as causal evidence that the seeds index separate reasoning manifolds.

\begin{tcolorbox}[
  colframe=gray,
  colback=gray!10,
  title=Question 2 (Code),
  breakable
]

You will be given a question (problem specification) and will generate a correct Python program that matches the specification and passes all tests.

\medskip
\textbf{Question:}

Takahashi has a piano with 100 keys arranged in a row.
The $i$-th key from the left is called key $i$.

He will play music by pressing $N$ keys one by one.
For the $i$-th press, he will press key $A_i$, using his left hand if $S_i = L$, and his right hand if $S_i = R$.

Before starting to play, he can place both of his hands on any keys he likes, and his fatigue level at this point is $0$.

During the performance, if he moves one hand from key $x$ to key $y$, the fatigue level increases by $|y-x|$ (conversely, the fatigue level does not increase for any reason other than moving hands).

To press a certain key with a hand, that hand must be placed on that key.

\medskip
\textbf{Task:}
Find the minimum possible fatigue level at the end of the performance.

\medskip
\textbf{Input}

The input is given from Standard Input in the following format:
\begin{verbatim}
N
A_1 S_1
A_2 S_2
......
A_N S_N
\end{verbatim}

\textbf{Output}

Print the minimum fatigue level at the end of the performance.

\medskip
\textbf{Constraints}
\begin{itemize}
  \item $1 \leq N \leq 100$
  \item $1 \leq A_i \leq 100$
  \item $N$ and $A_i$ are integers.
  \item $S_i$ is \texttt{L} or \texttt{R}.
\end{itemize}

\medskip

\end{tcolorbox}

\begin{tcolorbox}[
  colback=green!5,
  colframe=green!50!black,
  title=Response 2.1 (under seed 3 with reward 1),
  breakable
]

Okay, so I've got this problem to solve. Let me read it carefully and understand what's being asked.

The problem is about a piano with 100 keys. Takahashi is going to press N keys in a sequence. Each press is either with his left hand (L) or right hand (R). The initial placement of both hands can be anywhere, and the goal is to minimize the total fatigue. Fatigue increases when moving a hand from one key to another, specifically by the distance moved each time.

So, the task is to find the minimal possible total fatigue after all the presses.
It seems like a dynamic programming problem because each step depends on the previous state (where each hand is placed). The state of the system after each step is determined by the positions of the left and right hands.

Let's think about the states. At each step $i$, the left hand is on some key $x$, and the right on some key $y$. The initial state is $x$ and $y$ can be any keys, but initially, the fatigue is 0.

......

So, the DP transition is as follows:

For each step $i$ (from 1 to $N$):
For each previous state $(x_{\text{prev}}, y_{\text{prev}})$:
If $S_i$ is `L`, then the new state is $(A_i, y_{\text{prev}})$ with cost increase $|A_i - x_{\text{prev}}|$.
Otherwise, the new state is $(x_{\text{prev}}, A_i)$ with cost increase $|A_i - y_{\text{prev}}|$.

After processing all steps, the minimum value over all states gives the answer.

......

\medskip
\textbf{Code:}
\fontsize{9}{10}\selectfont %
\begin{verbatim}

import sys

def main():
    N = int(sys.stdin.readline())
    s_list = []
    for _ in range(N):
        a, s = sys.stdin.readline().split()
        s_list.append((int(a), s))

    INF = float('inf')
    dp = [[INF] * 101 for _ in range(101)]

    # Initialize step 0: all (x, y) have cost 0
    for x in range(1, 101):
        for y in range(1, 101):
            dp[x][y] = 0

    for i in range(N):
        a, s = s_list[i]
        new_dp = [[INF] * 101 for _ in range(101)]

        for x_prev in range(1, 101):
            for y_prev in range(1, 101):
                if dp[x_prev][y_prev] == INF:
                    continue
                current_cost = dp[x_prev][y_prev]

                if s == 'L':
                    new_x = a
                    new_y = y_prev
                    cost_inc = abs(new_x - x_prev)
                else:
                    new_x = x_prev
                    new_y = a
                    cost_inc = abs(new_y - y_prev)

                new_cost = current_cost + cost_inc
                if new_cost < new_dp[new_x][new_y]:
                    new_dp[new_x][new_y] = new_cost

        dp = new_dp

    min_cost = INF
    for x in range(1, 101):
        for y in range(1, 101):
            if dp[x][y] < min_cost:
                min_cost = dp[x][y]

    print(min_cost)

if __name__ == '__main__':
    main()
\end{verbatim}

\end{tcolorbox}

\begin{tcolorbox}[
  colback=green!5,
  colframe=green!50!black,
  title=Response 2.2 (under seed 8 with reward 1),
  breakable
]

Okay, so I have this problem to solve, and I need to find the minimum possible fatigue level after playing a sequence of piano keys. Hmm, let me read the problem statement again to make sure I understand it correctly.

So, Takahashi has a piano with 100 keys. Each key press is either with his left or right hand. For each press, if he uses a hand, it has to be placed on the key before pressing. The fatigue increases by the absolute difference of the keys when moving the hand. The initial fatigue is zero. I need to find the minimal total fatigue after all the presses.

The question is to determine the initial positions of his hands and then track the movements to minimize the total distance moved.

......

To solve this problem, we need to determine the minimum possible fatigue level after playing a sequence of piano key presses using either the left or right hand. The fatigue level is determined by the distance the hands move between consecutive key presses.

......

\medskip
\textbf{Code:}
\fontsize{9}{10}\selectfont %
\begin{verbatim}
n = int(input())
left = []
right = []

for _ in range(n):
    a, s = input().split()
    a = int(a)
    if s == 'L':
        left.append(a)
    else:
        right.append(a)

def calculate(arr):
    total = 0
    for i in range(1, len(arr)):
        total += abs(arr[i] - arr[i-1])
    return total

left_fatigue = calculate(left)
right_fatigue = calculate(right)

print(left_fatigue + right_fatigue)
\end{verbatim}

\end{tcolorbox}

\subsection{Usage of LLMs}\label{LLM}

LLMs are employed exclusively for grammar checking in this manuscript. They are not utilized for writing assistance, literature retrieval, or research ideation.

\newpage

\end{document}